\documentclass{ieeeaccess}
\usepackage{cite}
\usepackage{amsmath,amssymb,amsfonts}
\usepackage{graphicx}
\usepackage{textcomp}
\usepackage[inline]{enumitem}
\usepackage{array}
\usepackage{pifont}
\usepackage{ifpdf}
\usepackage{flushend}
\usepackage{enumitem}
\usepackage{adjustbox}
\usepackage{lipsum}
\usepackage{hhline}
\usepackage{mathtools}
\usepackage{pbox}
\usepackage{tablefootnote}
\usepackage[multiple]{footmisc}
\usepackage{amsthm}
\usepackage{graphicx}
\usepackage{algpseudocode}
\usepackage{algorithm}
\usepackage{comment}
\newtheorem{theorem}{Theorem}
\usepackage[flushleft]{threeparttable}
\usepackage{url}
\usepackage{makecell}
\begin{document}
\history{Date of publication xxxx 00, 0000, date of current version xxxx 00, 0000.}
\doi{10.1109/ACCESS.2017.DOI}

\title{A Broad Ensemble Learning System for Drifting Stream Classification}
\author{\uppercase{Sepehr Bakhshi}\authorrefmark{1},
\uppercase{Pouya Ghahramanian}\authorrefmark{2},
\uppercase{Hamed Bonab\authorrefmark{3}, and Fazli Can
}\authorrefmark{4}}
\address[1]{Department of Computer Engineering, Bilkent University, Ankara, Turkey  (e-mail: sepehr.bakhshi@bilkent.edu.tr)}
\address[2]{Department of Computer Engineering, Bilkent University, Ankara, Turkey, (e-mail: ghahramanian@bilkent.edu.tr)}
\address[3]{Manning College of Information and Computer Sciences, 140 Governors Dr., Amherst, MA 01003
(e-mail: bonab@cs.umass.edu)
}
\address[4]{Department of Computer Engineering, Bilkent University, Ankara, Turkey, 
(e-mail: canf@cs.bilkent.edu.tr)}
\tfootnote{This study is partially supported by TÜBİTAK grant no. 120E103.}

\markboth
{Author \headeretal: Preparation of Papers for IEEE TRANSACTIONS and JOURNALS}
{Author \headeretal: Preparation of Papers for IEEE TRANSACTIONS and JOURNALS}

\corresp{Corresponding author: Fazli Can (e-mail: canf@cs.bilkent.edu.tr).}
.

\begin{abstract}
In a data stream environment, classification models must handle concept drift efficiently and effectively. Ensemble methods are widely used for this purpose; however, the ones available in the literature either use a large data chunk to update the model or learn the data one by one. In the former, the model may miss the changes in the data distribution, and in the latter, the model may suffer from inefficiency and instability. To address these issues, we introduce a novel ensemble approach based on the Broad Learning System (BLS), where mini chunks are used at each update. BLS is an effective lightweight neural architecture recently developed for incremental learning. Although it is fast, it requires huge data chunks for effective updates, and is unable to handle dynamic changes observed in data streams. Our proposed approach named Broad Ensemble Learning System (BELS) uses a novel updating method that significantly improves best-in-class model accuracy. It employs an ensemble of output layers to address the limitations of BLS and handle drifts. Our model tracks the changes in the accuracy of the ensemble components and react to these changes. We present the mathematical derivation of BELS, perform comprehensive experiments with 20 datasets that demonstrate the adaptability of our model to various drift types, and provide hyperparameter and ablation analysis of our proposed model. Our experiments show that the proposed approach outperforms nine state-of-the-art baselines and supplies an overall improvement of 13.28\% in terms of average prequential accuracy.

\end{abstract}


\titlepgskip=-15pt

\maketitle

\section{Introduction}
\label{sec:introduction}
\PARstart{V}{arious} data stream sources generate an immense amount
of data in the blink of an eye. Social media, IoT devices, and sensors are all examples of such sources. The "3 V's of Big Data Management" summarizes the hurdles in this field. These are the Volume of the data, Variety which refers to numerous data types, and Velocity, which is one of the major problems in handling data streams due to its fast data arrival rate \cite{laney20013d}. Due to these hurdles, building models
that are capable of learning in streaming environments is a challenging
task. The developed methods require an approach specifically designed for this task, as it faces problems different from traditional machine learning. 

A major issue in data stream mining is \emph{concept drift} which refers to changes in the probability distribution of data over time \cite{widmer1996learning,lu2018learning,gama2014survey,gama2004learning}. If the change affects the decision boundaries it is referred to as \emph{real drift}, and if the distribution is altered without affecting the decision boundaries, it is called \emph{virtual drift} \cite{ramirez2017survey}. Concept drift is mainly categorized into four types: \emph{Abrupt}, \emph{Gradual}, \emph{Incremental}, and \emph{Recurring} \cite{lu2018learning}. In abrupt drift, the concept is suddenly altered to a new one which usually results in a prompt accuracy decline; however, gradual drift refers to replacement of the old concept with a new one in a gradual way. 
In incremental drift, an old concept changes to a new one incrementally over a period of time. The main difference between gradual and incremental drift is that in the former one, the class distribution is also prone to changes; however, in the latter, the values of the data are changed during a span of time \cite{iwashita2018overview}. Recurring drift occurs when an old and previously observed concept replaces the current one. 

\par
In terms of their implementation, data stream classification models in the literature could be categorized into two main approaches: \emph{chunk-based methods} and \emph{online learning} \cite{ditzler2015learning}. In a chunk-based approach, a fixed-size chunk of data is typically collected to update the model as new data items arrive; however, in an online learning approach, one data point at a time is used to update the system. Each approach has its pros and cons. Compared to online methods, a chunk-based approach learns faster at the start of the training because the model is seeded with a large initial chunk of data, making the update process more effective at the initial steps; however, a concept drift may occur later in the learning process, and if the drift is located within a chunk, it is possible that it will be missed. Such a concept drift may result in a significant reduction in model accuracy.
In the case of online learning, the latter problem does not affect the performance of the model; however, the model may suffer from slower initial learning performance \cite{li2020incremental}. Another problem of an online model is its runtime. Compared to the chunk-based models, online learning methods are less efficient and require more computations. Online models also suffer from instability as they learn the data one at a time \cite{li2020incremental}.

\textbf{Motivation.}
The disadvantages of using an online or chunk-based model motivates us to propose a method to alleviate the issues in these two approaches while keeping their positive features. Our main goal is to propose a resilient model that is able to function effectively and efficiently in a stream environment. To do so, we propose our solution based on Broad Learning System (BLS) \cite{chen2017broad}, a lightweight neural network. Using lightweight neural networks for concept drift adaptation is a relatively unexplored research avenue \cite{ijcai2022p788} and in this work, we aim to explore its potential for stream processing.
\par
BLS solves a much simpler least square equation for training the model in place of time-consuming loss calculations and back propagation. Although an incremental version of BLS is introduced in the original paper, BLS is not suitable yet to be used in a stream environment for the following two reasons: (i) In incremental mode, BLS is effective only when the chunk size is large; and (ii) The model has no mechanism to handle concept drift.
Since using large chunks may result in missing concept drifts inside a chunk, we use mini chunks (2 $\leq$ chunk size $\leq$ 50). Above all, to handle the problem of slow learning in initial steps of an online model, or a model that is trained with mini chunks, we adapt the BLS to learn with small chunk sizes and have a comprehensive feature mapping and output layer. We track the changes in the accuracy of the ensemble components, and remove, replace, or add them to the model at each time step. With this frequent exchange of the ensemble components between a pool of removed ensemble components and the ensemble itself, we create an ensemble of classifiers of good variety, as each of these ensemble components is trained with different data. This enables our model to react faster to concept drifts and in case of false removal of a component, it is quickly returned to the learning process. The same feature also helps our approach to be resilient. As mentioned by Vardi \cite{vardi2020efficiency}, “\textit{resilience via distributivity and redundancy is one of the great principles of computer science}”, and this principle is one of the main reasons that ensemble approaches function effectively in stream environment, and compared to a single classifier, are more resilient to changes in the data distribution. This redundancy (using several classifiers in an ensemble) decreases the efficiency of the model. Given that learning with mini chunks and using an ensemble both have a negative impact on the efficiency of the model in terms of computational cost, we only use output layers of the Broad Learning System as our ensemble components, and use a single feature mapping and enhancement layer. Meaning that the ensemble consists of a part of BLS, and not the whole BLS model. This results in a significant reduction in computational burden. In depth technical aspects of our ensemble model is demonstrated in the following sections.\par

\textbf{Contributions. }Our main contributions are the following. We

\begin{itemize}[leftmargin=*,align=left]
   \item {Design and mathematically derive an enhanced version of BLS suitable for a stream environment, trained with small chunks of data;}
   
   \item Propose an efficient and effective passive ensemble approach for concept drift handling based on tracking the changes in the accuracy of each ensemble component, and utilizing the output layers of the BLS as the ensemble components;

   \item Conduct experiments on 20 datasets with various concept drift types, and compare our results with nine state-of-the-art baselines.
\end{itemize}
\par
 In the upcoming sections, we first review the related works in the literature in Section \uppercase\expandafter{\romannumeral2\relax}. 
 Next, we explain our proposed approach in detail in Section \uppercase\expandafter{\romannumeral3\relax}. Then the experimental design is presented in Section \uppercase\expandafter{\romannumeral4\relax}. We report our experimental results, and compare our model with the baselines 
 in Section \uppercase\expandafter{\romannumeral5\relax}. We conclude our work and specify a future direction in Section \uppercase\expandafter{\romannumeral6\relax}.

\section{Related Work}
In this section, we study the proposed approaches for data stream classification from three different perspective. We begin by reviewing the methods \textit{for concept drift adaptation} in the literature, then we present the ensemble approaches and categorize them into \textit{active vs. passive}, and \textit{chunk-based vs. online} algorithms \cite{ditzler2012incremental}.
\label{sec:guidelines}
\par
\textbf{Concept Drift Adaptation. }For concept drift adaptation, two approaches are mainly studied in the literature. As mentioned by Gama \emph{et 
al.}\cite{gama2014survey} these two \textit{model management} techniques are \textit{single classifiers} \cite{ditzler2015learning} and \textit{ensemble methods} \cite{bonab2019less}.

A single classifier is accompanied with a concept drift detector. The detection algorithm utilizes alerts to warn the classifier of a drift. To find a drift point, detection methods usually follow the error rate, and trigger an alarm in the case of an unusual decline in overall accuracy. DDM\cite{gama2004learning}, EDDM\cite{baena2006early},  ADWIN\cite{bifet2007learning}, and OCDD\cite{gozuaccik2021concept} use this strategy. 
Another approach is to keep track of the statistical changes in the data distribution. PCA-CD\cite{qahtan2015pca}, EDE\cite{gu2016concept}, and  CM\cite{lu2016concept} are designed based on this method. After an alarm is triggered, the classifier usually restarts learning from that point on. kNN and the Hoeffding Tree\cite{hulten2001mining} are among the most popular classifiers used for this purpose. All the above-mentioned methods are supervised; however, unsupervised detection of concept drift is also studied by Gözüaçık \emph{et al.} in a recent work \cite{gozuaccik2019unsupervised}. Another work on unsupervised concept drift detection, but this time on multi-label classification, is presented by Gulcan \emph{et al.}\cite{gulcan2022unsupervised}.
\par
Models that use a single classifier are efficient; however, retraining the model from scratch or replacing it with a new classifier results in a delay in the learning process, since the model loses useful information learned so far; furthermore, in the case of a recurrent drift, the model should learn an already learnt concept from scratch.
\par

\par
In ensemble-based models a combination of learners is used to make the final decision. For concept drift adaptation, ensemble methods use various techniques; however, there are a few similar strategies that most of them utilize to maintain their high effectiveness during the learning process. For instance, they may have an adaptive strategy that adds and removes the classifiers based on their performance. Some ensemble methods preserve the old classifiers that are removed from the ensemble \cite{gama2014survey}, and utilize these classifiers in the later stages of the learning process based on the needs of the model. This approach helps the ensemble model to have an effective reaction to concept recurrence. Our proposed approach uses this strategy, but
one of the drawbacks of preserving the old classifiers is its storage and computational burden. Using a limit for the number of preserved classifiers is a simple yet effective strategy that we use for alleviating this issue. Another considerable advantage of declaring a limit for the number of classifier components of an ensemble is that controlling the ensemble size provides a consistent runtime during stream classification. The computational load of the model increases substantially as the number of classifier components increases, and this leads to an
 inefficient and sometimes broken system that is unable to
 process new data.
\par

Apart from the brief description of ensemble methods presented above, we classify them from two perspectives. Firstly, the way they handle concept drift. In this sense, ensemble approaches fall into two subgroups: \textit{Active and Passive}. Secondly, the data processing criteria they use. Here there are basically two choices \textit{Chunk-Based and Online}.

\textbf{Active vs. Passive Ensemble Methods. }
  Active methods rely on a concept drift detection method to trigger an alarm. Then, the ensemble model reacts to this drift by adding new classifiers, updating them, or restarting their learning process. Adaptive Random Forest (ARF)\cite{gomes2017adaptive},  Leverage bagging\cite{bifet2010leveraging},  Adaptive Classifiers Ensemble (ACE)\cite{nishida2005ace}, Heterogeneous Dynamic Weighted Majority (HDWM)\cite{idrees2020heterogeneous},  comprehensive active learning method for multi-class imbalanced streaming data with
concept drift (CALMID)\cite{liu2021comprehensive}, and Streaming Random Patches (SRP) \cite{gomes2019streaming} are among the well-known algorithms in this category.
 \par
 In passive models however, no concept drift detector is used, and usually a weighting strategy helps the model to adapt to the new changes \cite{krawczyk2017ensemble}. The model may also add or remove classifiers from the ensemble. With this method, the ensemble relies on the most recent data to make a prediction \cite{ditzler2015learning}. Additive Expert Ensembles (AddExp) \cite{kolter2005using},  Dynamic Weighted Majority (DWM)\cite{kolter2007dynamic}, Geometrically Optimum and Online-Weighted Ensemble (GOOWE)\cite{bonab2018goowe}, Learn++.NSE \cite{elwell2011incremental}, Resample-based Ensemble Framework for Drifting Imbalanced Stream (RE-DI)\cite{zhang2019resample}, and Kappa Updated Ensemble (KUE) \cite{cano2020kappa}
 fall into this category .
 \par

\par
\textbf{Chunk-Based vs. Online Ensembles. } 
In a chunk-based ensemble, a chunk of data is collected before each update. Using large chunks of data causes some problems like missing the actual drift point. This leads to delayed response in case of a concept drift, and decreases the accuracy of the model. In contrast to online models, chunk-based models are more efficient as they process a chunk of data simultaneously instead of processing the data one by one. Accuracy Weighted Ensemble (AWE)\cite{wang2003mining}, Learn++.NSE \cite{elwell2011incremental}, and Accuracy Updated Ensemble (AUE)\cite{brzezinski2011accuracy} are in this category.
\par
Unlike the chunk-based models, online ensembles process each data item separately, and they do not need to gather a chunk of data, and this leads to less memory usage. On the other hand, processing a single data instance at a time is time consuming, and may result in instability of the model. Meaning that small changes in the input may negatively affect the learning power of the model. Diversity of Dealing with Drift (DDD) \cite{minku2011ddd}, DWM\cite{kolter2007dynamic}, AddEXP  \cite{kolter2005using} are examples for online ensembles.

\par

 \begin{figure*}[h]
    \centering
    \includegraphics[ scale = 0.29]{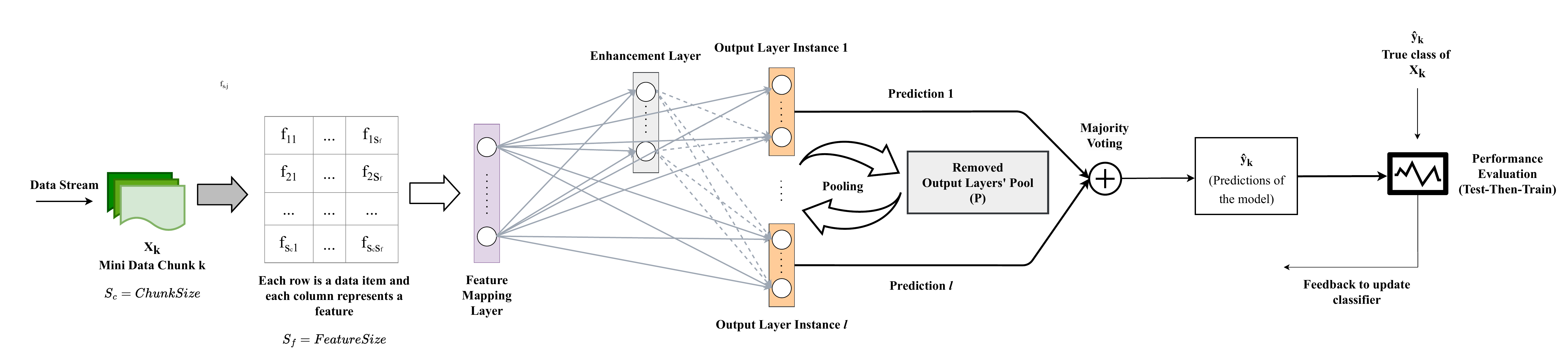}
    \caption{Overall schema of Broad Ensemble Learning System (BELS) for drifting stream classification.}
    \label{fig:BELS_figure}
\end{figure*}

\section{Method}
This section is organized as follows. First we introduce the problem that we aim to address. Next, we provide a brief overview of the BLS approach. Finally, we present the technical details of our proposed solution, which we call BELS.
 \subsection{Problem Definition}
In a data stream environment, data is continuously generated via a source over time. We refer to a data chunk at time step $k$ as $X_k$, and define the problem of data stream classification as follows: First, the model generates the probability values for each class in order to predict the class of incoming data chunk $X_k$.
Assuming that there are $C$ predefined classes, $s_k$ is the probabilities vector with the length of the number of labels. For each instance, we assume that the correct label becomes available at time step $k+1$, making the prequential evaluation \cite{gama2009issues} process possible. The model is then incrementally updated using the correct labels of $X_k$ and the obtained features. Fundamentally, this process is the assumption used in the majority of the data stream classification studies, and is known as the test-then-train learning paradigm in the literature \cite{gama2009issues}. As a result, every data instance is used both for training and testing. We propose BELS as being specifically designed for this problem setting. In the following subsections, we introduce our method step by step.
 \par
 
\subsection{BLS: Broad Learning System}

Broad Learning System (BLS) \cite{chen2017broad} achieves high performance in static environments, both in terms of runtime and accuracy. In its architecture, number of layers is minimized and instead, each layer consists of several number of nodes. In the first stage, BLS takes the input data and creates a feature mapping layer using a linear function. Then an activation function is used in the enhancement layer, and the output of the feature mapping layer is fed to it.
Next, the output of both the feature mapping layer and the enhancement layer is concatenated and fed to the output layer. To determine the output weights, BLS solves the least squares problem by finding the pseudoinverse. 
\par
 In the feature mapping layer, the data is mapped using  $\phi_i(X{W_e}_i + {\beta_e}_i)$ function, where ${W_e}_i$ and ${\beta_e}_i$ are random weights and bias of the $i$th mapped feature respectively. ${W_e}_i$ and ${\beta_e}_i$ are initiated randomly. Each mapped feature is used to form a group of $n$ mapped features by concatenation, where $n$ is the maximum number of feature mapping nodes in each group \cite{chen2017broad}.

 This concatenation is denoted as $Z^n = [ Z_1, ..., Z_n]$ \cite{chen2017broad}. 
 Next, the output of this feature mapping layer is fed to an activation function and form the enhancement layer. Furthermore, $i$th feature map is used to form the $j$th enhancement node $H_j$ as follows: $\xi(Z^i {W_h}_j + {\beta_h}_j)$ \cite{chen2017broad}. Like the feature mapping layer, the concatenation of $m$ enhancement nodes are grouped as $H^m = [ H_1, ..., H_m]$ \cite{chen2017broad} where $m$ is the maximum number of enhancement nodes.
Based on the original BLS paper\cite{chen2017broad}, the broad model is defined as follows:
\par

\par
$Y = [Z_1,...,Z_n|\xi(Z^n {W_h}_1 + {\beta_h}_1) ,... ,\xi(Z^n {W_h}_m + {\beta_h}_m)]W $
\par
$= [Z_1,...,Z_n|H_1,...,H_m]W$
\par
$= [Z^n|H^m]W$
\par

The ultimate goal of the system is to find $W$ by solving a least square problem. 
\par
BLS proposes an incremental way of updating the system for large datasets. The incremental approach uses a large chunk of data to update the system at every step; however, there are three main reasons that make BLS an ineffective system while dealing with data streams: i) BLS uses the initial set of incoming data for generating the feature mapping. The problem is that in each time step, the system uses the same data representation without any update. ii) By using a smaller chunk size, the proposed pseudoinverse updating in BLS is not effective, since it needs a very large chunk of data for an effective update. iii) There is no mechanism to handle concept drift in BLS and its proposed variants. Note that our proposed model is different from other variations in \cite{chen2018universal}.  In \cite{chen2018universal}, the proposed variations of BLS are not designed for a stream environment, and can not handle concept drift.
Furthermore, our model utilizes an ensemble approach and focuses on using small chunk sizes for tracking the changes in the accuracy of the ensemble components. 
\par

\subsection{BELS: Broad Ensemble Learning System}
To tackle the problems mentioned in Section \uppercase\expandafter{\romannumeral3\relax}.B we propose BELS.
In Sections \uppercase\expandafter{\romannumeral3\relax}.C.1 and \uppercase\expandafter{\romannumeral3\relax}.C.2 we propose a solution for considering the whole data items for generating the feature mapping, enhancement, and output layers, which prepares the model for effective updates with smaller chunk sizes ($S_c$$\leq$ $1000$). In Section \uppercase\expandafter{\romannumeral3\relax}.C.3 we introduce our ensemble approach for dealing with concept drift. 
\par
\algnewcommand\algorithmicinput{\textbf{Input:}}
\algnewcommand\algorithmicoutput{\textbf{Output:}}
\algnewcommand\Input{\item[\algorithmicinput]}%
\algnewcommand\Output{\item[\algorithmicoutput]}%

  Figure \ref{fig:BELS_figure} shows the overall structure of BELS. As we can see in this figure, when the first chunk of data is received, it is fed to the feature mapping layer. Then output of feature mapping layer is fed to the enhancement layer. Subsequently, the output of these two layers are concatenated and fed to an ensemble of output layers, and the final decision is made based on the majority voting. We utilize a pool of removed output layers where the ones removed from the ensemble are kept in it. Frequent exchange of the removed output layer instances in the pool and the active ones in the ensemble is used as an strategy for handling concept drift. Detailed explanation of each step is presented in the following subsections. 
 
\subsubsection{Updating the Feature Mapping in BELS}
    BLS employs a sparse autoencoder to overcome the randomness of the generated features. We use the same method to deal with this issue; however, despite BLS, we update this feature representation after each chunk of the data.   \par
   To obtain this feature representation, BLS uses an iterative method. Eq. \eqref{eq2} is defined in BLS paper for this purpose \cite{chen2017broad}. The output of these iterative steps is denoted as $\mu$. 
    
\begin{flalign}
    \begin{cases}
      w_{i+1} := \left(z^T z+\rho I\right)^{-1} \left(z^T X_k + \rho \left(o^i - u^i \right) \right)\\
      o_{i+1} := S_{\lambda/\rho} \left(w_{i+1} + u^i \right)\\
      u_{i+1} := u^i + \left(w_{i+1} - o_{i+1} \right)
    \end{cases}&&
    \label{eq2}
\end{flalign}
\par

In Eq. \eqref{eq2}, $X_k$ is the input data in time step $k$ and ${z}$ is the projection of the input data using $XW_{e_i} + \beta_{e_i}$ for the  \textit{i}th feature group. $W_e$ and $\beta_{e}$ are randomly generated with proper dimensions initialized at the beginning of the first time interval. Note that we apply the same $W_e$ and $\beta_{e}$ during the training for each update. $u^i$, $o^i$, and $w^i$ are initialized as zero matrices at the beginning of the iteration. These matrices are only used to update purpose in iterative steps of Eq. \eqref{eq2}. In the formula, $\rho >0 $, $I$ is the identity matrix, and $S$ is the soft thresholding operator \cite{chen2017broad}. In Eq. (\ref{eq:s}), (a) is the sum of $w_{i+1}$ and $u^i$. In our experiments we use $\kappa = 0.001$. $S$  is calculated as follows \cite{chen2017broad}: 
\begin{equation}
S_\kappa(a) = 
\begin{cases}
    a - \kappa   ,\quad  a > \kappa \\
    0 ,\quad  \quad  |a| <= \kappa\\
    a + \kappa ,\quad  a < -\kappa
\end{cases}
\label{eq:s}
\end{equation}
\par

While considering the incremental input in a stream environment, the projection of data is different for each chunk of data. Meaning that we can not utilize the first set of data to calculate the $\mu$, and use the same $\mu$ for the rest of the incoming data. Additionally, if the data chunk is small, then this projection is not comprehensive enough. To solve this problem, we propose an updating system for Eq. \eqref{eq2}, that in each step $k$, ${\mu}^{i}$ is the projection of the entire data from step one to step $k$. In each time step, the system uses the renewed $\mu$, and this helps the model to have a comprehensive feature mapping layer that represents the entire data up to that time step. This technique improves the accuracy of the model drastically when dealing with both small and large chunks.
\par
To implement this idea, we need to update $z^{T}X_k$ and $z^T{z}$ in each time step for every set of mapping features in $k$, separately.
\par

Let us denote $z^{T}X_k$ as $T_1$. We know that dimensions of $T_1$ are independent of the number of data instances in each time step, and depend on the number of feature mapping nodes, as well as the number of features in each data instance. Based on this fact, we can use the following theorem in our method:

\begin{theorem}\label{foo}
For two matrices $A$ and $A^\prime$ with the same number of columns, if we multiply $A^T$ with $A$, and ${A^\prime}^T$ with $A^\prime$, the result of both multiplications are square matrices with the equivalent size. We refer to them as $A_t$ and $A^\prime_t$, respectively. Let us concatenate $A$ and $A^\prime$ vertically and denote the new matrix as $A_c$. The product of $A_c^T$ and $A_c$ is a square matrix equal to sum of $A_t$ and $A^\prime_t$. The hypothesis can be formulated as follows:
\begin{equation}
    A_c^T A_c = A^T A + A^{\prime T} A^\prime \quad  \text{where:} \quad {A_c} = 
    \begin{bmatrix} 
        A &|& A^\prime
    \end{bmatrix}
    \label{eqn:maineqn}
\end{equation}

\end{theorem}

\begin{proof}

    Let $A$ be an \textit{m} by \textit{n} matrix and let $A^\prime$ be an $m^\prime$ by n matrix. $A_c$ is obtained by concatenating $A$ and $A^{\prime}$ matrices vertically as follows:

    \begin{equation*}
    \centering
        A_{c} =
        \begin{bmatrix}
            a_{11} & a_{12} & \dots & a_{1n} \\
            a_{21} & a_{22} & \dots & a_{2n} \\
            \vdots & \vdots & \ddots & \vdots \\
            a_{m1} & a_{m2} & \dots & a_{mn} \\
            a^{\prime}_{11} & a^{\prime}_{12} & \dots & a^{\prime}_{1n} \\
            a^{\prime}_{21} & a^{\prime}_{22} & \dots & a^{\prime}_{2n} \\
            \vdots & \vdots & \ddots & \vdots \\
            a^{\prime}_{m^{\prime}1} & a^{\prime}_{m^{\prime}2} & \dots & a^{\prime}_{m^{\prime}n} \\
        \end{bmatrix}
    \end{equation*}

    Left-hand side of the Eq. \eqref{eqn:maineqn} is the multiplication of $A_c^T$ and $A_c$ matrices. Let us denote the element at i\textit{th} row and j\textit{th} column of the resulting matrix as $l_{ij}$. Also, let the (i, j)\textit{th} element of the resulting matrix on the right-hand side of the Eq. \eqref{eqn:maineqn} be represented by $r_{ij}$. In the following, we show that these two elements are equal regardless of i and j values, meaning that the resulting matrices in the left and right-hand sides of the Eq. \eqref{eqn:maineqn} are equal. \\
    \begin{equation*}
        \begin{gathered}
            l_{ij} = \sum_{k=1}^{m+m^{\prime}} A_c^{ki} \times A_c^{kj} = \sum_{k=1}^{m} A_c^{ki} \times A_c^{kj} + \sum_{k=m+1}^{m+m^{\prime}} A_c^{ki} \times A_c^{kj} \\
            = \sum_{k=1}^{m} a_{ki} \times a_{kj} + \sum_{k^{\prime}=1}^{m^{\prime}} a_{k^{\prime}i} \times a_{k^{\prime}j}
        \end{gathered}
    \end{equation*}
    \begin{equation*}
        r_{ij} = (A^T A)^{ij} + (A^{\prime T} A^\prime)^{ij} = \sum_{k=1}^{m} a_{ki} \times a_{kj} + \sum_{k^{\prime}=1}^{m^{\prime}} a_{k^{\prime}i} \times a_{k^{\prime}j}
    \end{equation*}       

\end{proof}

Based on Theorem 1, Eq. \eqref{eq:1} is used to update $T_1$ value. 

\begin{equation}
   T_{{1}_{k}} = \sum_{i=1}^{k}T_{{1}_i}
\label{eq:1}
\end{equation}
\par
Let us denote $z^T{z}$ as $T_2$. Assuming that the number of columns in $z$ does not change during the training phase, based on Theorem 1, we define the Eq. \eqref{eq3} to update $T_2$ at each step $k$:
\par
\begin{equation}
   T_{{2}_{k}} = \sum_{i=1}^{k} T_{{2}_i}
   \label{eq3}
\end{equation}
\par
We use $T_{2_k}$ and $T_{1_k}$ as the input for the modified version of Eq. \eqref{eq2}, and utilize it as in Eq. \eqref{eq5}.
\begin{flalign}
    \begin{cases}
      w_{i+1} := \left(T_{2_k}+\rho I\right)^{-1} \left(T_{1_k} X_k + \rho \left(o^i - u^i \right) \right)\\
      o_{i+1} := S_{\lambda/\rho} \left(w_{i+1} + u^i\right)\\
      u_{i+1} := u^i + \left(w_{i+1} - o_{i+1} \right)
    \end{cases}&&
    \label{eq5}
\end{flalign}

\par
The enhancement nodes are created using the following formula \cite{chen2017broad}:
\par
\begin{equation}
  H_j = [\tanh({{Z_k}^n} W_{h_j} + \beta_{h_j}) ]
   \label{eq6}
\end{equation}
\par

${Z_k}^n$ is a set of feature mapping nodes at time step $k$, and $W_{h_j}$ and $\beta_{h_j}$ are generated randomly. It's noteworthy that enhancement layer is used as it is proposed in \cite{chen2017broad}.

\par
While updating the model, the random weights are fixed and the enhancement nodes are updated at each step. The process of updating the feature mapping and enhancement layers are shown in Algorithm 1.
\par
After concatenating the output of the feature mapping layer and the enhancement layer horizontally,  the next step is calculating the pseudoinverse for the least square problem discussed in BLS \cite{chen2017broad}. 
\par

\begin{algorithm}
\caption{Feature mapping and enhancement layer update}

\begin{algorithmic}[1]
\Input  {Data chunk $X_k$}
\Output {A set of feature mapping and enhancement nodes denoted as $A_k$
}
\State{Initiate random $W_e$, $W_h$, $\beta_e$ and $\beta_h$ at the beginning
  }
\State{$X_k$ = data instances at step $k$}

\For{i=0 ; i $\leq$ n }
    \State{$z =  (X_kW_{e_i} + \beta_{e_i})$}
    
    \State{$T_1  ={X_k} {z}^T $}
    \State{${T_2}  ={z}^T z$}
    \If {$k=0$}
        
         \State{${T_{1_{k}}}  = T_1$}
    
        \State{${T_{2_{k}}}  = T_2$ }
       
    \Else
        \State{${T_{1_{k}}}  ={T_{1_{k-1}}} + T_1$}
    
        \State{${T_{2_{k}}}  ={T_{2_{k-1}}} + T_2$} 
    
   \EndIf
    \State{Calculate $\mu_i$ with Eq. \eqref{eq5}}
    \State{$Z_i= X\mu_i$}
   \EndFor
    \State{Set the feature mapping group ${Z_k}^n = [Z_1,...,Z_n]$}
   
    \For{j $\leftarrow$ 1 ; j $\leq$ m }
    
        \State{calculate $H_j = [\tanh({{Z_K}^n} W_{h_j} + \beta_{h_j}) ]$ with Eq. \eqref{eq6}}
    \EndFor
     \State{Set the enhancement node group ${H_k}^m = [H_1,...,H_m]$\par
     $A_k = [{Z_k}^n | {H_k}^m]$ }

\end{algorithmic}
\end{algorithm}

\subsubsection{Updating the Pseudoinverse in BELS}
 Unlike BLS, where the pseudoinverse is calculated based on the instances of the last chunk of data, we revise this calculation such that it represents the pseudoinverse of the whole data until that time step.
\par

Suppose that $A$  is the result of concatenating the output of the feature mapping layer and output of the enhancement layer horizontally. 
To obtain the pseudoinverse, BLS uses Eq. \eqref{eq7} \cite{chen2017broad}:
\begin{equation}
   W = {(\lambda I + A{A^T})}^{-1} {A^T}Y
   \label{eq7}
\end{equation}

where $Y$ is the labels of a chunk of training data and $\lambda$ is a small positive value added to the diagonals of $A$.
To calculate weights of the output layer, which is considered as the solution for the least square problem, we update $A{A^T}$ at each time step $k$ and refer to it as ${A}_{t}$ :
\begin{equation}
   {A}_{t} = {{A^T}_{k}}A_{k}
\end{equation}

Then, we multiply ${A^T}_{k}$ by $Y_k$, which is the set of labels at time step $k$ and define it as follows:
\begin{equation}
   D_t = A^T_k Y_k
\end{equation}
\par
 ${A_{t}}_{k}$ and ${D_{t}}_{k}$ values are obtained as:
\begin{equation}\label{eq:9}
  A_{{t}_{k}} = \sum_{i=1}^{k} A_{{t}_i } \quad\text{and}\quad 
   D_{{t}_{k}} = \sum_{i=1}^{k} D_{{t}_i}
\end{equation}
Based on theorem 1, we know that at the end of step $k$, ${A_{t}}_{k}$ and ${D_{t}}_{k}$ are equal to ${A}_{t}$ and ${D_{t}}$ of the entire data until that time step, respectively. The process is shown in Algorithm 2. First we calculate ${A_{t}}_{k}$ and ${D_{t}}_{k}$ (Algorithm 2:1-9). Then, Eq. \eqref{eq10} is used to update the pseudoinverse (Algorithm 2:10).
\begin{equation}
     W = {\left(\lambda I + A_{t_k}\right)}^{-1} D_{t_k}
     \label{eq10}
\end{equation}

\par 

For testing, we use a similar approach to BLS \cite{chen2017broad}; however, just like the previous steps, we separate the calculations related to feature mapping and enhancement layers, and output layer. Let us assume that ${A_{test}}_{k}$ is the concatenation of test results of feature mapping and enhancement layers. We use ${A_{test}}_{k}$ for producing the prediction based on the following formula in the final stage \cite{chen2017broad}:  
\begin{equation}
     \hat{y} = {{A_{test}}_{k} W}
     \label{eq14}
\end{equation}

\begin{algorithm}
\caption{Output layer weight calculation}

\begin{algorithmic}[1]
\Input  {A set of feature mapping and enhancement nodes denoted as $A_k$}
\Output {$W$}

\State{$A_t = A_k^T A_k$}
\State{ $D_t = A_k^T Y_k$}
\If {$k=0$}
        
     \State{${A_{t_{k}}}  =A_{t} $}

    \State{$D_{t_k}  = D_{t} $ }
   
\Else
    \State{${A_{t_{k}}}  ={A_{t_{k-1}}} +A_{t} $}

    \State{$D_{t_k}  = D_{t_{k-1}} +D_{t} $} 
\EndIf
\State{$W = (\lambda I + A_{t_k}) ^ {-1} D_{t_k}$}
\end{algorithmic}
\end{algorithm}

\begin{algorithm}[t]
\caption{BELS (Broad Ensemble Learning System) }
\begin{algorithmic}[1]
\Require $D$ : data stream,  $X_k$:  data chunk at step k,               $Y_k$: labels of the data chunk at step k, $\delta$: accuracy threshold

    \Ensure 
            ${\hat{y}}_k$: prediction of the ensemble as a score vector at step $k$
    \While{$D$ has more instance}

        \If {$\xi$ is not full}\par
        \State {
           $\xi$ $ {\leftarrow}$ add new output layer
        }
        \ElsIf{$\xi$ is full \textbf{or} $L$ length $>$ ($\xi$ length /2)}
        
        \For{i $\leftarrow$ 0 ; i $\leq$ $L$ length     - 1 }
        \State{    remove $\xi$$[$L$[i]] $}\par
        \If{L length $>$ ($\xi$ length/2)}
            \State{$P$ $\leftarrow$ $\xi$$[L[i]] $}
            
        \EndIf
        \EndFor
        \While{$\xi$ is not full and $i$ $<$  $bP$ length}
        
        \State{$\xi$ $\leftarrow$  $bP[i] $}
        \State{ remove $bP[i]$ from $P$} 
        
        \EndWhile
        \If{$P$ length $>$ $M_p$}
           
            \State{keep the last $M_p$ instances of $P$}
         
        \EndIf
    \EndIf
   
        \State{Calculate ${{A}_{test}}_{k}$\par}
        \For{i $\leftarrow$ 0 ; i $\leq$ $\xi$ length}
            \If{${O}_i$ is initialized}
                \State{${s_{i}}^k$,  ${acc_i}^k \leftarrow$     Prediction \& acc of $O_i$ Eq.\eqref{eq14}
                
            \If{chunk size!= 2 }
                
                \State{$\delta$= overall accuracy}
            \EndIf
                \If{ ${acc_i}^k$ $<$ $\delta$}
                
                    \State{$L$ $\leftarrow$ i}\par
                
                \EndIf
            \EndIf
        \EndFor
        \State{${\hat{y}}_k$ $\leftarrow$ Use the set of ${s_{i}}^k$ for hard voting} 
        \For{j $\leftarrow$ 0 ; j $\leq$ $P$ length  }
        {
             \State{ ${acc_j}^k$ $\leftarrow$  Test $P[j]$ using (Eq. \eqref{eq14})}
             \If { ${acc_j}^k$ $>$ $\eta$}
             {
                $bP$ $\leftarrow$ $P[j]$
             }
             \EndIf
        }
        \EndFor
        \State{$A_k$ $\leftarrow$ update $F$ and $E$ using Algorithm 1. }
        \For{i $\leftarrow$ 0 ; i$\leq$ $\xi$ length  }
            
            \State{update ${O}_i$ using Algorithm 2.} 
            
        \EndFor

\EndWhile
}
\end{algorithmic}
\end{algorithm}

\subsubsection{Concept Drift Adaptation in BELS}
For concept drift handling, we use an ensemble approach. Our approach is passive as we do not use any concept drift detection mechanism. Simply, we keep updating the feature mapping layer and enhancement layer as new data arrive; however, an ensemble of output layer instances is used to determine the final result. Each output layer instance is a collection of output layer nodes.
\par 
There are two main reasons for excluding the feature mapping and enhancement layers from our ensemble: 
\begin{enumerate*}[label=(\roman*)]
\item Using an ensemble of a whole BLS model is not efficient as it needs more calculations, \item Initializing the feature mapping and enhancement layer for each arriving data chunk delays the learning process, because the initial feature mapping and enhancement layers of the data is not comprehensive.
\end{enumerate*}
In our ensemble, only the output layer instances with the best prediction in the last chunk are kept in the model, and those with an accuracy less than threshold $\delta$ are replaced with a new one, or one of the output layer instances in the pool. The pool consists of removed output layer instances. Our approach to managing concept drift involves regularly swapping out the output layer instances that have been removed from the pool with the active ones in the ensemble. Using small chunk sizes results in more number of exchanges, and this helps our ensemble to have components trained with good variety of the data. Besides, it also reacts swiftly to any changes in the distribution of the data. Symbols and notations used in the ensemble learning process, and the detailed steps of our approach are shown in Table \ref{table:1} and Algorithm 3, respectively.
\par

\renewcommand{\arraystretch}{1.15}
\begin{table}
    \centering
    \caption{Additional Symbols and Notation for BELS (Algorithm 3) }
    \begin{tabular}{|c l |}
        \hline
        Symbol & Meaning  \\
        \hline 
       
        $\xi$& Ensemble of output layers
        $\xi  = \{{O}_1,{O}_2,{O}_2,..., {O}_l\}$\\
        \hline
        $P$ & The pool containing the removed output layers\\
        \hline

        $L$ & List of indexes of output layer instances with an \\& accuracy
        lower than threshold  (for each chunk)\\
       \hline
        ${s_{i}}^k$ &Relevance scores for the $i$th classifier for $k$th data \\&chunk in 
        
         the ensemble. ${s_{i}}^k = <{s_{i1}}^k,{s_{i2}}^k,{s_{i3}}^k,...>$\\
        \hline
        $bP$ & Output layer instances from $P$ which achieve an \\&accuracy   
        higher than a threshold in a chunk\\
        \hline
        $M_o$ & Maximum number of output layer instances \\
        \hline
        $M_p$ & Maximum number of output layer instances in $P$\\
        \hline
        $\delta$ & Accuracy threshold for removing an ensemble component \\

        \hline
    \end{tabular}
    \label{table:1}
\end{table}

Our model is defined with three independent but connected parts. Feature mapping layer denoted by $F$, enhancement layer denoted by $E$, and output layer denoted by $O$. BELS consists of a single $F$ and $E$, and an ensemble of $l$  output layer instances $\xi  = \{{O}_1,{O}_2,{O}_2,..., {O}_l\}$.
To update these parts at each time step, Algorithm 1 and Algorithm 2 are used.
 A set of new data instances $I = \{I_1,I_2,I_3,...,I_{S_c}\}$ where $2 \leq S_c \leq 50$ is used for each update.
 \par
Let us assume that we have received the first chunk of data and our model is initialized. In the next iterations, first, we check if the number of output layer instances has reached the predefined maximum number (Algorithm 3: 2-3). If $\xi$ is full, then the model removes the  output layers instances that have an  accuracy less than a threshold. If the number of output layer instances in $L$ is more than the instances in $\xi$, then the output layer instances in $L$ are added to $P$ (Algorithm 3:4-10). Next, the model adds the previously removed output layer instances from $bP$ back to the model (Algorithm 3:11-14). $bP$ Consists of output layer instances that were removed once and now are eligible to return to the learning process. Then we check the size of the pool, and if it passes the predefined threshold ($M_p$), we only keep the last $M_p$ output layer instances in the pool.

In the testing phase, the accuracy of each output layer instance is calculated and the index of the worst-performing ones are added to $L$ (Algorithm 3:20-30).   

 Hard voting is used for calculating the final prediction. In hard voting, the score-vector of an output layer instance $s^k$
is first transformed into a one-hot vector, and then combined with the $s^k$ of the other output layer instances to determine the final prediction (Algorithm 3:31).
Next, output layer instances in $P$ are tested. If any of them has an accuracy more than a threshold (denoted by $\eta$), then that output layer instance is added to $bP$ (Algorithm 3:32-36). This process gives the output layer instances in $P$ a chance to be added to $\xi$ and used in the learning process as of the next time step. Finally, the model is updated (Algorithm 3:38-40) and the same processes are repeated\footnote{We will make the implementation of BELS available on github after the review process.}.

\subsubsection{Time Complexity Analysis}
Since we assume that the calculations for building each feature mapping and enhancement node takes $O (1)$ time, let us assume that building the feature mapping layer and enhancement layer takes $O (n)$ and $O(m)$ time respectively, where $n$ is the number of feature mapping nodes, and $m$ is the number of enhancement nodes. Based on this assumption, we conclude that Algorithm 1 and first phase of testing (which includes generating of the feature mapping and enhancement) take $O(n+m)$ time. Next, we have our ensemble method in Algorithm 3. Execution time for initializing the model and removing the output layer instances, or adding them back to the model (Algorithm 3:2-18) 
 are negligible.

 Let us denote the ensemble size as $S_\xi$. Then, the final prediction for each output layer instance is calculated in (Algorithm 3:20-30), and it takes $O(S_\xi)$ time.

 Later, we test the output layer instances in the pool in (Algorithm 3:32-36). This process takes $O(P)$ time. Then for training, we first update the feature mapping and enhancement layer using Algorithm 1. This process is executed once for each chunk, and it takes $O(n+m)$ time. Finally, we update the output layer instances using Algorithm 2. In this algorithm, we calculate the pseudoinverse. Let us denote chunk size as $S_c$. Based on the dimensions of ${D_t}_k$ and ${A_t}_k$, we conclude that Algorithm 2 takes 
$O(max(S_c,(n+m))^2 min(S_c,(n+m)))$
 time.

\par
For $\frac{N}{S_c}$ chunks of data, the whole process complexity is as follows:
\begin{equation}
\label{time_complexity}
\begin{gathered}
O\bigg(\frac{N}{S_c}\Big(2(m+n)+S_\xi+P +
\big( S_\xi \, max(S_c,(n+m))^2\\ min(S_c,(n+m))\big) \Big)\bigg)
 \end{gathered}
\end{equation}
  
Obviously, among different parts of the algorithm, the one with $O( \xi \, max(S_c,(n+m))^2 min(S_c,(n+m)))$ time complexity dominates the whole process. The final complexity of the algorithm is as follows:
\begin{equation}
     O\Big(\frac{N}{S_c}\big(S_\xi \, max(S_c,(n+m))^2 min(S_c,(n+m))\big)\Big) 
   \label{time_complexity_2}
\end{equation}

\par

The analysis shows that our model's efficiency in terms of processing time is agnostic of the feature set size, and depends on the chunk size and the number of nodes. This ability allows the user of the model to set the hyperparameters based on the needs of the project in a real-world environment. Efficiency of the baseline methods heavily relies on the size of the feature set, and it significantly decreases as the feature set size increases. For instance, in two variations of the Waveform dataset (\textit{Waveform} and \textit{Waveform-Noisy}), we can observe that the runtime of baseline models is significantly higher in the dataset with a larger feature set size (\textit{Waveform-Noisy}); however, the runtime of BELS is not affected by the increase in the number of features (See Table \ref{tab:time_results}).
\par
\section{Experimental Design}
In this section, first the datasets and baseline methods are introduced, then we elaborate on our experimental setup.

 \renewcommand{\arraystretch}{1.12}
\begin{table}
\centering
\caption{Properties of 20 Datasets used in the experiments}
\begin{threeparttable}
\begin{tabular}{|l|r|r|r|r|r|}

\hline
Dataset &    Type &   DT & $\#$ F  &   $\#$ C &     Size  \\
\hline

Airlines & Re & U  & 7 & 2 &  539,383 \\
\hline
CoverTypes  & Re & U & 54 & 7 &  581,012 \\
\hline
Electricity\tablefootnote{
\url{https://github.com/ogozuacik/concept-drift-datasets-scikit-multiflow}}\cite{harries1999splice}  & Re &       U    & 8 & 2 &  45,312 \\
\hline

Email\textsuperscript{1}\cite{katakis2010tracking}  &  Re &      A\&R    & 913 & 2 &  1,500   \\
\hline
Phishing\textsuperscript{2}\cite{sethi2017reliable}   & Re &      U   & 46 & 2 &  11,055 \\
\hline

Poker \textsuperscript{2}\cite{Dua:2019}   & Re &       U    & 10 & 10 &  829,201  \\
\hline
Spam\textsuperscript{2}\cite{sethi2017reliable}   & Re &       U       & 499 & 2 &  6,213\\

\hline
Usenet\tablefootnote{
\url{http://mlkd.csd.auth.gr/concept_drift.html}\label{footnote_1}} \cite{katakis2008ensemble}
   & Re &A\&R & 99 & 2 &  1,500  \\
\hline
Weather\textsuperscript{2}\cite{elwell2011incremental}  & Re & U & 8 & 2 &  18,152 \\

\hline
Interchanging RBF\textsuperscript{2}\cite{losing2016knn}     & Syn &        A & 2& 15&  200,000  \\
\hline

LED-drift& Syn & A   & 7 & 10 &  100,000 \\
\hline
MG2C2D\tablefootnote{
\url{https://sites.google.com/site/nonstationaryarchive/datasets}\label{footnote_3}}\cite{dyer2013compose}& Syn &       I\&G  & 2 & 2 &  200,000 \\
\hline

Moving Squares\textsuperscript{2}\cite{losing2016knn}     & Syn &       I& 2 & 4 &  200,000 \\
\hline
Rotating Hyperplane\textsuperscript{2}\cite{losing2016knn}& Syn &       I  & 10 & 2 &  200,000 \\

\hline
SEA-Abrupt-01210\tnote{*}   & Syn & A\&R & 3 & 2 &  100,000 \\
\hline
SEA-Abrupt-13123\tnote{*}   & Syn & A\&R   &3 & 2 &  100,000\\
\hline
SEA-Incremental-01210\tnote{*}   & Syn & I\&R  & 3 & 2 &  100,000 \\
\hline
SEA-Incremental-13123\tnote{*}    & Syn & I\&R  & 3 & 2 &  100,000 \\
\hline
Waveform   & Syn & -  & 21 & 3 &  100,000 \\
\hline
Waveform-Noisy\tnote{*}    & Syn & -  & 40& 3 &  100,000 \\
\hline

\end{tabular}
\begin{tablenotes}
  \item[*] Noisy datasets are marked with (*).
  
\end{tablenotes}
\end{threeparttable}

\label{tbl:datasets}
\end{table}

\renewcommand{\arraystretch}{1.45}
\begin{table*}[t]
    \centering
    \caption{Baseline methods with a short description}
    \begin{threeparttable}
    \begin{tabular}{|l | l |}
        \hline
        Baseline Name & Short Description\tnote{*}  \\
        \hline 
       
        AddExp \cite{kolter2005using}& Additive Expert Ensemble utilizes a pruning strategy by removing the weakest components. \\
        \hline
        DWM \cite{kolter2007dynamic} & Dynamic Weighted Majority uses an ensemble of weighted classifiers. Weights are assigned based on the performance of classifieirs.\\
        \hline
        GOOWE \cite{bonab2018goowe} & Geometrically Optimum and Online-Weighted Ensemble uses geometrical approach for assigning weights for the classifiers. \\
       \hline
         HAT\cite{bifet2009adaptive} &Hoeffding Adaptive Tree Classifier uses an ADWIN concept drift detector for detecting the drift and pruning the tree.\\
        \hline
        
        KUE \cite{cano2020kappa} & Kappa Updated Ensemble combines online and chunk-based approaches using Kappa statistic to update the weights of classifiers. \\
        \hline
        LevBag \cite{bifet2010leveraging} & Leveraging Bagging Classifier is an active ensemble approach built on top of the Oza Bagging method. \\
        \hline
        OzaADWIN \cite{oza2001online} & Oza Bagging ADWIN Classifier utilizes a concept drift detector for enhancing the performance of Online Bagging Algorithm. \\
        \hline
        SAM-kNN \cite{losing2016knn}& The Self Adjusting Memory (SAM) is a kNN-based ensemble method for data stream classification.\\
        \hline
         SRP \cite{gomes2019streaming} & Streaming Random Patches is an ensemble method that utilizes both bagging and random subspaces.\\
        
        \hline
    \end{tabular}
    \begin{tablenotes}
    \item[*] Descriptions of each method is based on the information on Scikit-mutliflow \cite{JMLR:v19:18-251} web page or their original paper.
    \end{tablenotes}
    \end{threeparttable}
    \label{tab:baselines}
\end{table*}

\subsection{Datasets}
In our experiments, we use 20 datasets to evaluate the efficiency (in terms of runtime) and effectiveness of our model, and compare it with the baselines. Our datasets cover a wide spectrum of possibilities that can be observed in data streams. Nine Real (Re) and 11 Synthetic (Syn) datasets are used. Their properties are presented in Table \ref{tbl:datasets}. All four drift types are used in the experiments: Gradual (G), Incremental (I), Abrupt (A), and Recurring (R). In the same table, (U) stands for unknown drift type.
 For each dataset, the Drift Type, number of features, and number of class labels are denoted as (DT), (\# F), and (\# C); respectively.
 \par
 First four synthetic datasets in Table \ref{tbl:datasets} are from other works in the literature \cite{dyer2013compose,losing2016knn}, and the LED, SEA and Waveform datasets are generated using the scikit-multiflow library \cite{JMLR:v19:18-251}. In the LED dataset, seven drifting features are used without noise. For the SEA dataset, each number in the the dataset names represents one of the classification functions (There are four classification functions: $\{0,1,2,3\}$). At each drift point, a new classification function is used to generate the stream. We introduce five drift points in each dataset. SEA datasets are also noisy. Abrupt and Incremental drifts contain 10\% and 20\% noise, respectively. For adding noise to the datasets, the labels of the dataset change from 0 to 1 and vice versa, according to a probability distribution \cite{JMLR:v19:18-251}. The Waveform-Noisy dataset adds 19 irrelevant attributes as noise to the dataset \cite{JMLR:v19:18-251}.

\newcommand{\cmark}{\ding{51}}%
\newcommand{\xmark}{\ding{55}}%

\subsection{Experimental Setup}
The evaluation is based on the \textit{interleaved-test-then-train} approach \cite{gama2009issues}. It is the most common technique for evaluating classification models in a data stream environment. In this approach, a data instance or a chunk of instances are first used for testing, then for training the model.
\par
One of the most challenging tasks in data stream mining is optimizing hyperparameters for neural networks. Recently, a potential straightforward solution has been proposed, which involves using a pool of neural networks with varying hyperparameters \cite{gunasekara2022online}. During the learning process, all neural networks in the pool are trained, and after receiving each data instance, the model with the lowest loss (excluding the current data instance) is used for prediction; however, in our case, such approach may result in a significant computational burden. To address this issue, we utilize a modified version of this method that only considers the first 1,000 data items to choose a set of hyperparameters for the rest of the learning process (for Usenet and Email datasets we use the first 100 data items because of their small size). For each dataset, we use this subset of data to perform a local grid search to determine the optimal number of feature mapping nodes, number of enhancement nodes, and chunk size. We then choose the combination of hyperparameters with the highest accuracy, and continue using the same set of hyperparameters during the learning process. We limit the number of nodes and chunk sizes to a predefined set to reduce the time needed for the grid search. Therefore, three combinations of BELS are used as follows: $BELS_1$ ($n$: 25, $m$: 1), $BELS_2$ ($n$: 25, $m$: 50), $BELS_3$ ($n$: 100, $m$: 100), where, as defined earlier, $n$ is the number of feature mapping nodes, and $m$ is the number of enhancement nodes. For chunk size, we use a set of five different chunk sizes: $\{2,5,10,20,50\}$. 
\par
The other three hyperparameters are set to default when the model is initiated: $M_o$= $75$, $M_P$= $300$, $\eta$= $0.5$, and $\delta$ is updated dynamically in the model (See Algorithm 3:23-25). These hyperparameters are tuned based on a grid search on several values, and the ones that contribute positively to both runtime and accuracy are chosen as default.

\par

\par

\par

To compare the performance of BELS with other methods, we choose nine state-of-the-art models as baselines. The names of the baseline models and a short explanation about their approach is supplied in Table \ref{tab:baselines}. All the baselines (except GOOWE and KUE) are implemented using scikit-multiflow \cite{JMLR:v19:18-251} library. The source codes for GOOWE and KUE are available on their github\footnote{\url{https://github.com/abuyukcakir/goowe-python}} and website\footnote{\url{https://people.vcu.edu/~acano/KUE/}}. To have a fair comparison, we use default configuration for BELS and baselines. For the baseline methods, default values introduced in their papers are utilized in the experiments.

\par

The experiments are conducted on a PC with Intel(R) Xeon(R) Gold 5118 CPU @ 2.30 GHz and 128 GB RAM on an Ubuntu 18.04.4 LTS operating system.

\par
\section{Experimental Results and Evaluation}
In this section, we first compare our method with its variants and also with the original BLS. Then we perform a thorough experimental evaluation with various baselines in terms of prequential accuracy and runtime. Next, we discuss the results of statistical significance analysis. Finally, the results of hyperparameter sensitivity analysis are discussed.

\newcolumntype{M}[1]{>{\centering\arraybackslash}m{#1}}

\newcolumntype{P}[1]{>{\centering\arraybackslash}p{#1}}
\renewcommand{\arraystretch}{1.05}
\begin{table*}[t]
    \centering 
    \caption{{Comparison between BELS, its variants, and BLS in terms of average prequential accuracy and runtime (runtime is reported in centiseconds for processing of 1,000 data items). Accuracy improvement of BELS wrt. to BLS is provided in the last column.  Best results for each dataset is in bold }}

    \resizebox{\textwidth}{!}{
    \begin{tabular}{|l|r r |r r |r r| r r |r| }
         \hhline{~ --------~}
    \multicolumn{1}{l|}{}&
     \multicolumn{2}{c|}{ BLS}&
    
    \multicolumn{2}{c|}{ BELS-FPs}&
    
    \multicolumn{2}{c|}{ BELS-Ens}&
     \multicolumn{2}{c|}{ BELS}&
     \multicolumn{1}{c}{ }
    \\
    \hhline{----------}
        Datasets (DT) & Accuracy& Runtime  & Accuracy & Runtime & Accuracy& Runtime& Accuracy &Runtime &\pbox{50cm}{\% Acc. Imp.}  \\
        \hline
        Electricity (U) & 49.94 & \textbf{177.08}   & 60.54&180.79& 84.68 & 285.95& \textbf{87.21} & 597.06&  74.62\\
        
        Usenet (A \& R) & 54.76 & \textbf{13.29} & 59.71& 14.49 &\textbf{70.64} &16.58& 70.29 & 25.80&28.36\\

        MG2C2D (I \& G)& 61.37 & \textbf{17.67} & 91.06 & 22.81 &92.06 &  83.61&\textbf{92.94}& 301.65 &51.44 \\
        
        Rotating Hyperplane (I)& 81.54 & \textbf{6.42}  & 81.76 & 7.96&85.94 & 10.08&\textbf{90.79} & 54.66&11.34\\
        
         \hline
    \end{tabular}
    }
    \label{tab:ablation}
\end{table*}

\begin{figure*}[h]
    \centering
    \begin{tabular}{c c}
    \includegraphics[ scale = 0.46]{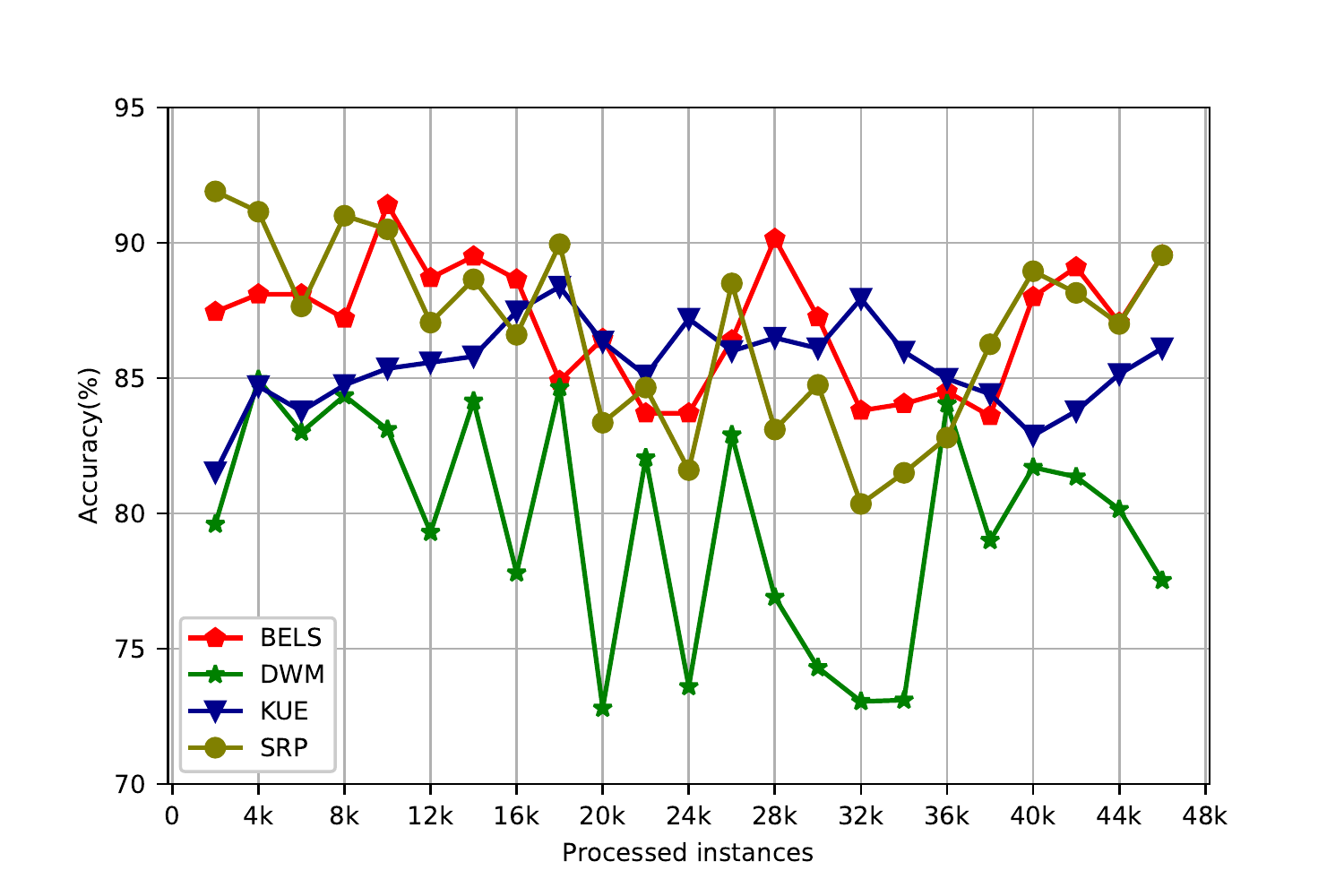} 
    
    &\includegraphics[ scale = 0.46]{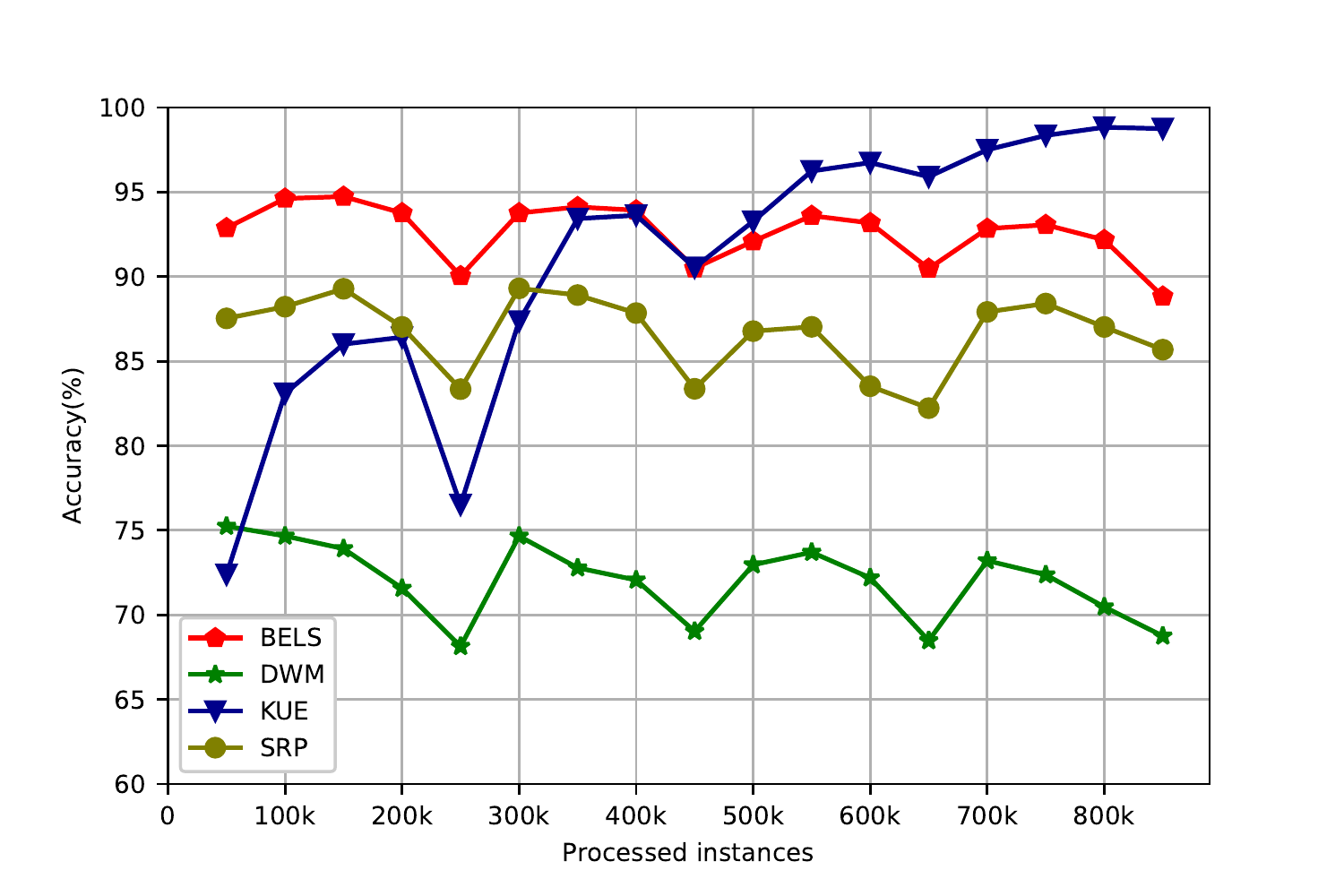}
    \\
    (a) Electricity (Unknown) & (b) Poker (Unknown)  \\

    \includegraphics[ scale = 0.46]{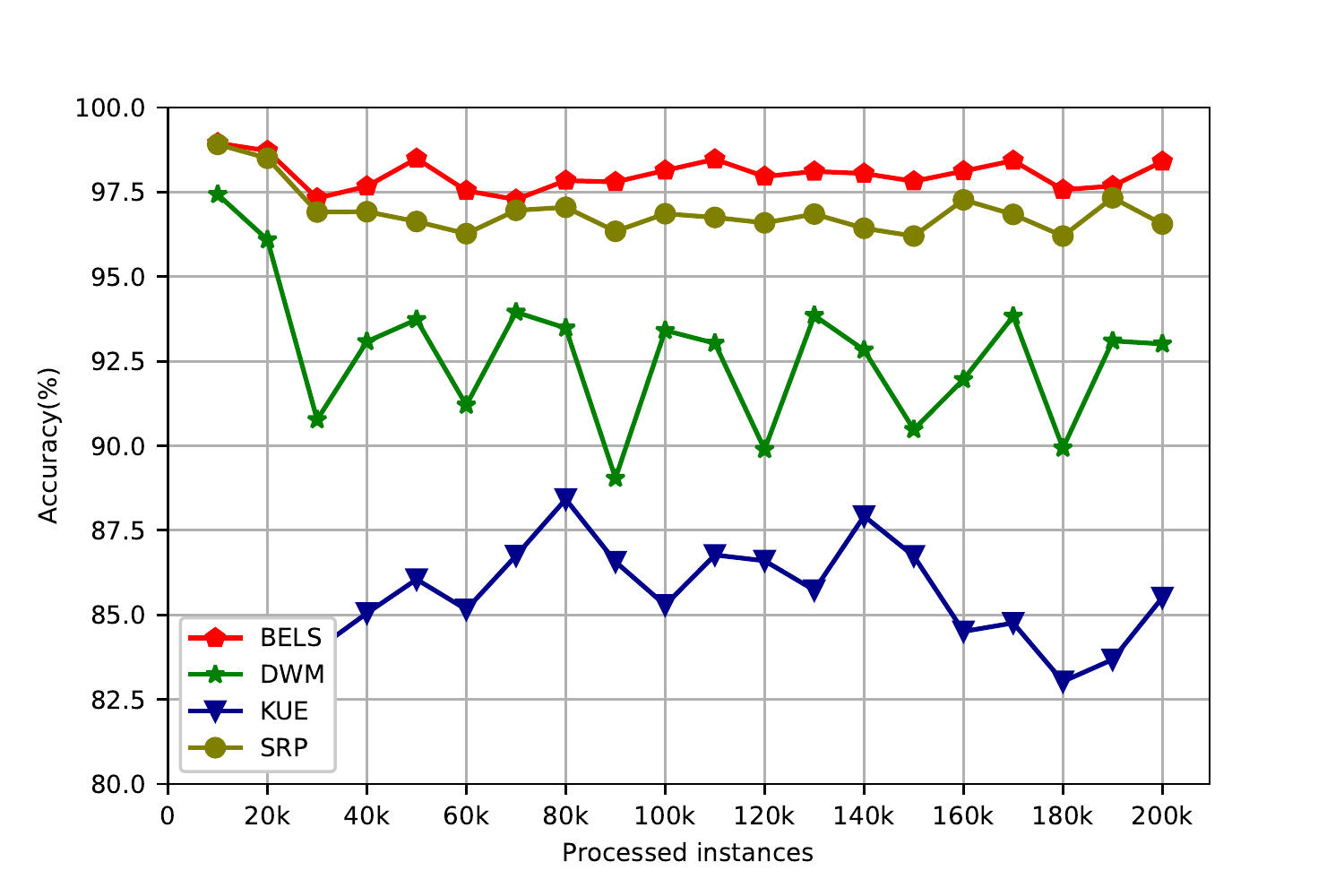} & \includegraphics[ scale = 0.46]{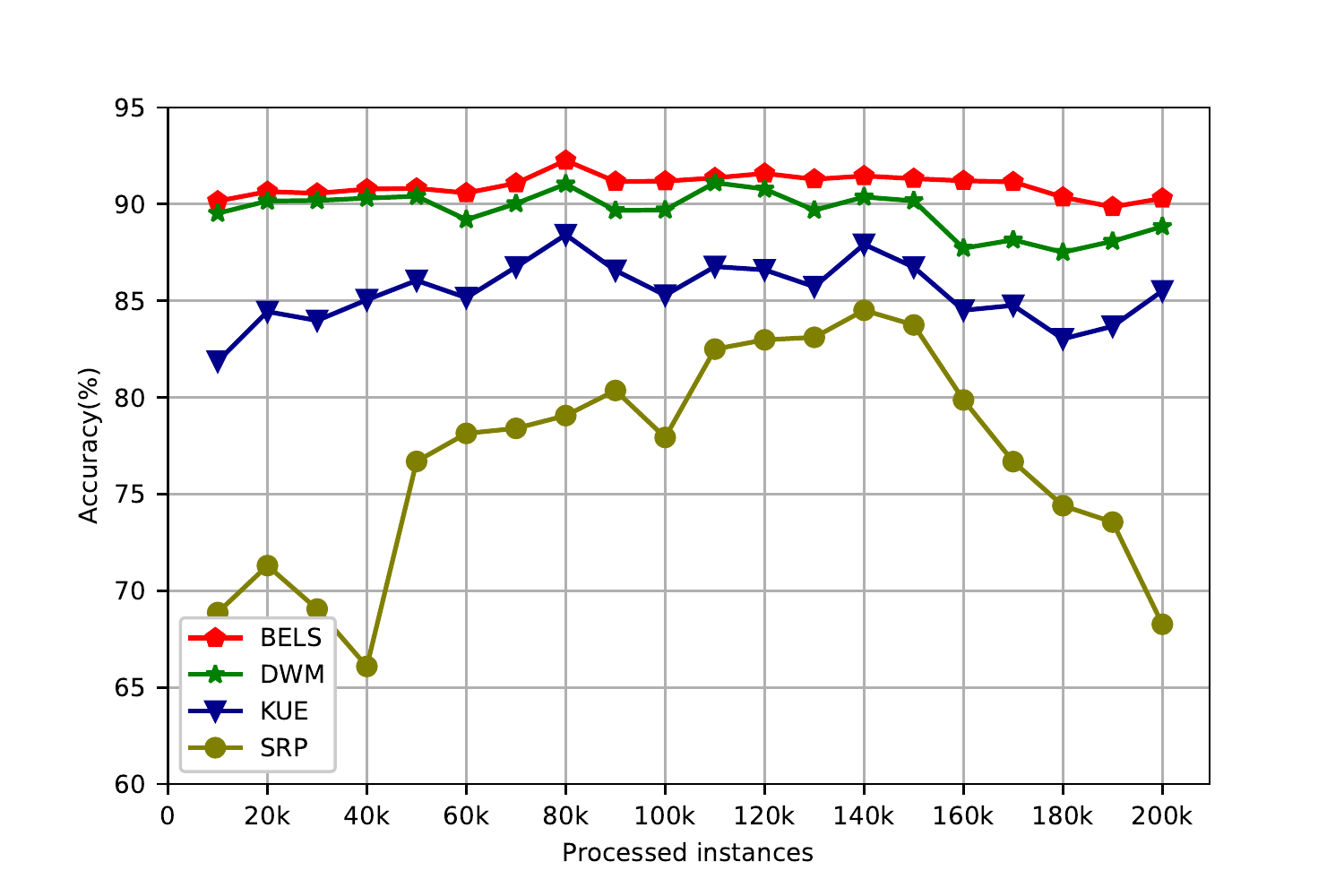} 

    \\
    (c) Interchanging RBF (Abrupt)  & (d) Rotating Hyperplane (Incremental)\\
\end{tabular}
    \caption{Prequential temporal accuracy results of BELS and three top ranked baselines (DWM, KUE, SRP) (drift types are given within parentheses).}
    \label{fig:my_label}
\end{figure*}

\subsection{BELS vs. BLS - Ablation Study}
To study the effect of our approach on the learning process and the handling of concept drift, we conduct a step-by-step study to show the improvement in average prequential accuracy and runtime. 
We analyze the effects of the different features of BELS on accuracy by performing the experiments on three different BELS variants, and compare it with the original BLS model in terms of runtime efficiency and effectiveness.

\begin{itemize}[leftmargin=*,align=left]
   \item{\textbf{BLS}}: The original Broad Learning System method.
    \item{\textbf{BELS-FPs}}: BELS method with the enhancements mentioned in Section \uppercase\expandafter{\romannumeral3\relax}.C.1 and \uppercase\expandafter{\romannumeral3\relax}.C.2, which includes feature mapping layer and pseudoinvese update.
    \item{\textbf{BELS-Ens}}: 
    BELS as an ensemble method. This part does not contain the pool of removed output layer instances.
    \item{\textbf{BELS}}: Complete BELS version with all of its features.
\end{itemize}

We report the results of four data streams of a good variety on four datasets.
 The chosen datasets contain all types of concept drift (incremental, gradual, abrupt, recurrence). To have a fair comparison, we use the same hyperparameters for all three variants of BELS and the BLS. The results are shown in Table \ref{tab:ablation}.
\par

BLS is not designed for a data stream environment. Based on this fact, we know that BLS is faster than our algorithm and at the same time, it has lower accuracy. As we see in Table \ref{tab:ablation}, by utilizing each feature, the accuracy improves, which yields the best performance in the complete version of BELS with an average accuracy improvement of 41.44\% on the four datasets used in Table \ref{tab:ablation}. Accuracy improvement is calculated as $\frac{\text{Accuracy Increase}}{\text{Original Accuracy}}\times 100$. Average accuracy improvement is the mean of improvement over all datasets.

\par

\renewcommand{\arraystretch}{0.93}
\begin{table*}[t]
    \centering
    \caption{Average Prequential Accuracy Results (in \%). For each row the highest value is marked with a bold text. Avg. \% Imp. by BELS wrt a baseline is the mean value of \% improvements obtained for individual datasets. Average Rank is the mean of rank of each method for individual datasets}
    \resizebox{\textwidth}{!}
    {
        \begin{tabular}{l r r r r r r r r r r}
            \hline 
            Datasets & BELS & AddExp & DWM &GOOWE& HAT&KUE& LevBag & OzaADWIN & SAM-kNN & SRP  \\
            
            \hline
            Airlines & 61.85 &58.44  &61.56 & 64.35 & 64.33 &\textbf{66.68}& 56.23 & 59.69   &58.26 &64.63 \\
            
            CoverType & 87.29  & 57.61 &80.38& 84.62 & 83.04& 87.32& \textbf{95.32}& 93.15  &93.20 &93.73\\
    
            Electricity & \textbf{87.21} & 73.96 & 79.79& 76.28& 83.27 &76.89&  83.49 & 76.15 &82.47 &86.64  \\
            
             Email & \textbf{68.36} & 55.26 & 56.48& 57.81& 54.48 & 52.70& 57.48 & 54.76 &63.04 &58.12  \\
             
            Phishing & 
            \textbf{93.74} 
            & 91.60 & 92.04& 90.58& 89.78 & 92.06&90.73 &92.30  &92.31 &92.43 \\
            
            Poker & \textbf{91.71} &58.98 & 72.10& 61.41& 66.64 &90.67&  81.44 & 82.23   &79.72 &86.70 \\
            
            Spam & \textbf{95.30} &90.31 & 89.41 & 86.41& 86.82  &82.36& 94.00 & 90.16 &91.78 &92.49 \\
            
            Usenet & \textbf{70.29} & 64.47 & 63.91 & 62.96& 65.47 & 66.30& 52.98& 57.12 &65.05&67.90  \\
            
            Weather & 77.36 & 69.10 & 70.07&71.52& 73.78 &74.95& 73.38 & \textbf{78.22}  &75.20&75.52\\
    
            Interchanging RBF & \textbf{98.20} &  17.02 & 92.70&69.15& 61.78 &69.30& 96.39 & 94.14  &95.81 &96.92 \\
             
             LED-drift& \textbf{100.00}& \textbf{100.00}& \textbf{100.00} & \textbf{100.00}& 99.97& 
             \textbf{100.00}&\textbf{100.00}&\textbf{100.00} & \textbf{100.00}&91.56 \\
            MG2C2D& 92.94&55.34 &91.94& 90.85& 92.70 &\textbf{92.95}&  90.47 & 92.55  & 91.93 &92.53\\
            
            Moving Squares & 89.99 & 32.66& 70.57& 34.15& 74.74 &30.27&60.73 & 43.01 &\textbf{97.37} &79.18  \\
            
            Rotating Hyperplane& \textbf{90.85} &81.88 & 89.64& 87.77& 85.77 & 85.44&73.44 & 82.21 &79.30 &76.82  \\

             SEA-Abrupt-01210& \textbf{88.65}& 86.94& 88.37 & 87.72& 87.56&  88.21&80.44 &87.54 &85.11 &86.11\\
             
             SEA-Abrupt-13123& \textbf{88.68}& 86.21 &88.29 & 87.83 & 87.89 & 87.77&80.00 & 87.23 &84.73 &85.31 \\
             
             SEA-Incremental-01210& \textbf{78.53}& 78.26& 78.38 & 78.24&78.17& 78.28&67.15& 75.37&72.71 &75.28 \\
             
             SEA-Incremental-13123& \textbf{78.53}& 77.79& 78.18 & 78.28& 77.85&  78.17&67.03&75.33 & 72.29&72.55\\

             Waveform& \textbf{84.96}& 80.08& 79.62 & 82.84 & 79.83& 83.67 &76.24&82.08 & 80.29&80.11\\
                    
             Waveform-Noisy& \textbf{85.77}& 80.70& 79.60 & 83.51&80.76& 83.26&73.80&80.94 & 78.85&78.43 \\
             
            \hline
            Avg. Accuracy&\textbf{85.51}&69.83	&80.15&76.81&	78.73&78.36 &77.54&	79.21&	81.97&	81.65	\\

            \hline
            Avg. \% Imp. by BELS& - & 46.66  & 7.36& 16.70& 9.86&  15.80 &11.98& 11.14  &4.80 & 5.05 \\
           \hline

            Avg. Rank& \textbf{1.78}& 7.53&5.58& 5.83& 6.10 & 4.78&7.18 & 5.63 & 5.73 & 4.90\\
            \hline
            
        \end{tabular}
    }
    \label{tab:acc_results}
    
\end{table*}

\renewcommand{\arraystretch}{1}
\begin{table*}[h]
    \centering
    \caption{{Average runtime for processing 1,000 data instances (in centiseconds). For each row the lowest value is marked with a bold text}}
    \resizebox{\textwidth}{!}{
    \begin{tabular}{l r r r r r r r r r r}
        \hline 
        Datasets & BELS & AddExp & DWM&GOOWE& HAT & KUE&  LevBag & OzaADWIN  & SAM-kNN & SRP \\
        \hline

        Airlines & 119.94 &159.76& 163.79 & 513.67 & 98.86 & 52.20&  651.25 & 561.39 & \textbf{17.63} & 432.46 \\
        
        CoverType & 716.79 &2,218.52 & 1,282.03 & 3,365.33 & 468.39 & 429.92&3,151.58 & 4,166.71 & \textbf{19.20} & 1,145.10 \\

        Electricity & 597.06 &143.74 & 136.94 & 282.60 & 66.50  & 26.70& 931.21 & 734.62 & \textbf{17.39} & 456.04\\

        Email &120.20 &6,694.00 & 5,557.33 & 2,043.33 & 2,285.33 & 462.67& 55,938.00 & 15,354.00 & \textbf{76.00} & 21,680.67 \\

        Phishing & 
        803.80&763.73 & 719.49 & 1,107.55 & 244.78 & \textbf{68.29}& 3,396.56 & 2,608.86& 130.80 & 1,004.07\\
        
        Poker &1,554.27 &629.47 & 316.07 & 835.01 & 146.32& 122.20& 957.41& 2,637.76 & \textbf{20.80} & 617.47 \\

        Spam & 7,004.02 &7,654.91 & 7,151.62 & 6,826.65& 2,794.46 & 516.01& 28,925.48 & 9,331.40 & \textbf{63.09} & 24,224.07\\
            
        Usenet & \textbf{25.80 }&762.67& 523.33 & 210.00 & 302.00 & 89.33&  5,168.67 & 1,290.67& 28.00 & 3,589.33 \\

        Weather & 65.17 &283.05 & 143.57 & 265.65& 71.73& \textbf{13.17}&  636.90 & 548.75 & 19.94 & 994.27\\

        Interchanging RBF & 321.03 & 197.01 & 79.18 & 477.24 & 76.46  & \textbf{12.86}& 356.70 & 1,205.40 & 322.04 & 1,524.79 \\
        
        LED-drift& 269.88&236.60 & 2,047.38 & 496.53 & 782.37 & \textbf{67.98}&3,924.38 & 19,355.88 & 1,348.22 & 3,475.63 \\
        
        MG2C2D&301.65 & 89.17 & 79.39 & 152.54 & 331.14 &\textbf{ 5.50}& 50.54 & 340.00 & 35.80 & 487.41 \\
        
        Moving Squares & 509.15 & 112.70 & 83.69 & 211.29 & 56.83  & \textbf{4.54}& 314.83 & 640.57 & 16.87 & 526.93 \\
        
        Rotating Hyperplane& 54.66 & 154.20 & 148.98 & 397.73 & 759.36 & \textbf{23.27}&  95.28 & 719.14 & 100.86 & 1,075.10 \\

        SEA-Abrupt-01210& 119.91 &76.56 & 75.89 & 273.81 & 58.40 & \textbf{10.64}&  465.92 & 417.18 & 31.12 & 307.37  \\

         SEA-Abrupt-13123& 126.51 &77.78 & 76.61 & 272.27 & 55.90 & \textbf{13.64}& 456.52 & 412.97 & 25.98 & 319.12 \\
         
        SEA-Incremental-01210& 112.99&77.62 & 77.06 & 272.49 & 59.73 & \textbf{11.39}& 467.09 & 419.04 & 20.54 & 375.14 \\
        
        SEA-Incremental-13123& 115.87&76.02 & 73.69 & 273.29 & 53.14  & \textbf{12.18}& 468.48 & 421.03 & 40.09 & 29.72 \\

         Waveform& 235.46& 232.37& 214.06 & 907.32& 112.70& 46.34& 1,399.12&2,028.46 & \textbf{39.96}&728.85\\
         
         Waveform-Noisy& 244.34& 410.31& 391.26 & 1,650.97& 181.50& 80.37& 2,432.25&3,494.95 & \textbf{64.42}&1,695.35\\

        \hline
        Avg. Runtime& 670.93&1,052.51&967.07&1,041.77&450.30&\textbf{103.46}&5,509.41&3,334.41&126.99&3,234.44\\
        \hline
        Avg. Rank&
        5.30&5.45&4.95& 6.10& 4.30& \textbf{2.20}& 7.90&8.15&3.30& 7.35 \\
        \hline
    \end{tabular}
    }
    \label{tab:time_results}
\end{table*}
\subsection{Effectiveness and Efficiency Analysis}

The prequential evaluation results are shown in Table \ref{tab:acc_results}. They demonstrate that our proposed method outperforms the baseline methods on 15 out of 20 datasets, and has a marginal difference with the best performing baseline in other five datasets. Our observations on datasets of varying drift types with different numbers of class labels and features demonstrate the adaptability of our model under different classification conditions. 
On average BELS improves the accuracy by 13.28\% compared to the baselines used in the experiments across 20 datasets.
Experiments on noisy datasets also show that our model is robust to noise.  
\par
We can see the runtime of the models in Table \ref{tab:time_results}. Runtime is considered as the processing time for test and train of 
 1,000 data instances in the interleaved-test-then-train method. The results show that our model has an average processing time of 670.93 centiseconds for processing 1,000 data items.

 \par

\par
The plots in Figure \ref{fig:my_label} show that BELS has robust concept drift-resistant performance: it maintains much better performance when concept drift occurs (indicated by the fluctuated accuracies as the data stream progresses). Furthermore, it provides a higher level of accuracy compared to the three top ranked baselines (KUE, SRP, DWM); to instantiate, in two synthetic datasets (Interchanging RBF and Moving Squares), it is always the top-performing model during the entire process. 
In the Electricity dataset, we observe abrupt fluctuations in the prequential accuracy. This suggests that there might be an abrupt drift in this dataset. In the Poker dataset, there are mild abrupt changes in the accuracies; however, in case of BELS, we see soft transitions in the plot that demonstrate the ability of BELS to handle concept drifts in such cases. The drifts in the other two datasets (Interchanging RBF and Rotating Hyperplane) are abrupt and incremental, respectively. In the Interchanging RBF dataset,  unlike BELS and SRP, the other two baselines (KUE and DWM) experience sharp fluctuations in prequential accuracy. In the Rotating Hyperplane however, the transitions of the accuracy plot are soft in all of the methods. These plots are typical, and we observe similar results for other datasets.
\par
To gain a better perspective of the average rankings of the methods used in our experimental evaluation, we demonstrate them in Figure \ref{fig:bar_chart}. 
 Each bar shows the average ranking of the models in terms of accuracy and runtime. 

\begin{figure}[H]
\centering

   \includegraphics[width=0.75\columnwidth]{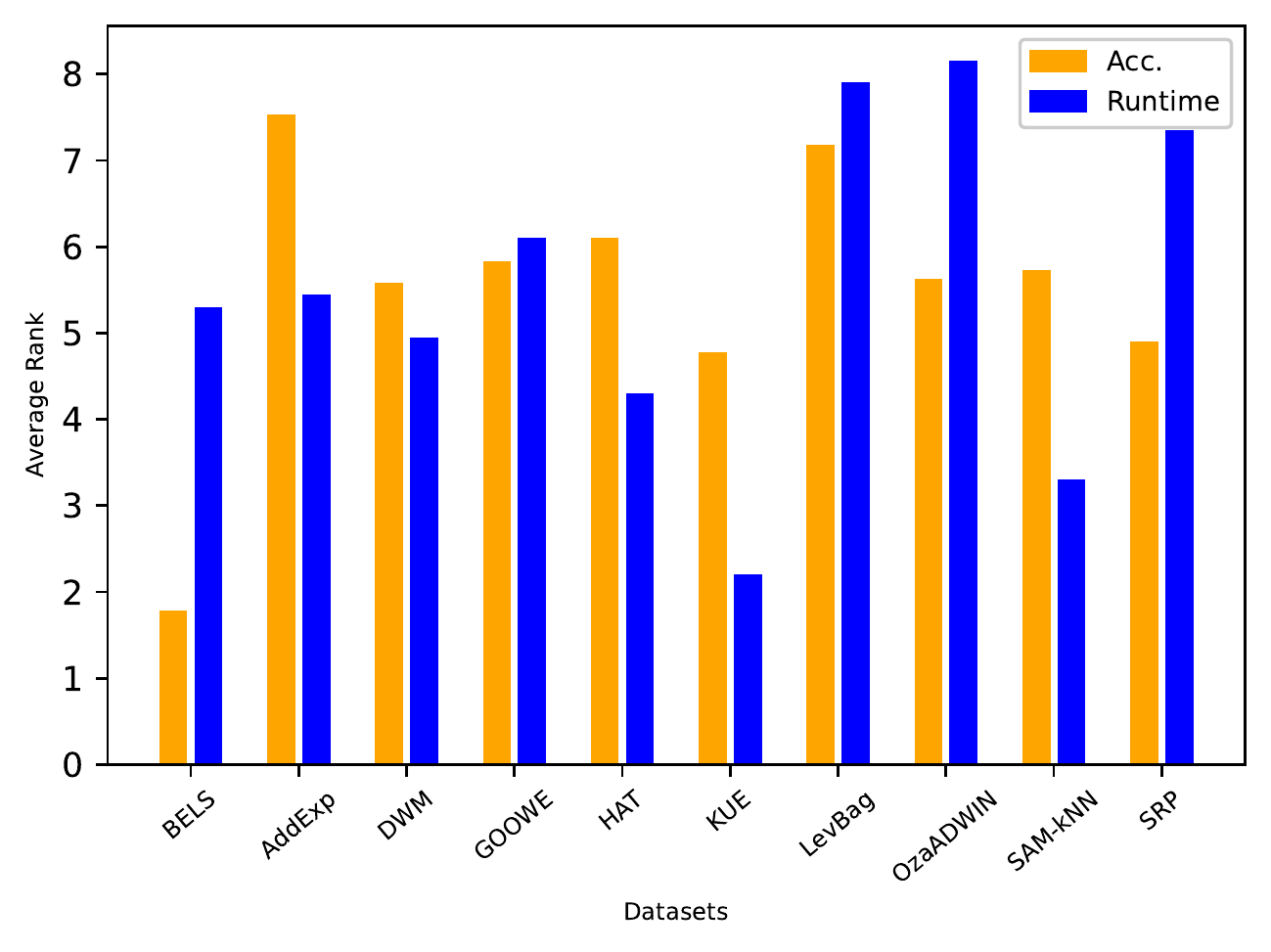}

\caption{Average rank comparison (lower is better) for prequential accuracy and runtime based on Tables \ref{tab:acc_results} and \ref{tab:time_results}.}
\label{fig:bar_chart}
\end{figure}

\subsection{Statistical Significance Analysis}
 In this part, we present a statistical test on both accuracy and runtime.
The analysis is conducted for 10 models and 20 datasets. By using the \textit{Friedman Test} we first reject the null hypothesis that there is no statistically significant difference between the mean values of the populations. Then we use \textit{post-hoc
Bonferroni-Dunn} test to see if there is a significant difference between the results of our proposed model and the baselines \cite{demvsar2006statistical}.
\par
For this test, we first rank the models based on their performance. Then based on the post-hoc Bonferroni-Dunn test, we calculate the \textit{Critical Difference} as $CD$= $2.65$. In our experiment $\alpha$$=$ $0.05$.
Figure \ref{fig:stat} shows the critical distance diagram for average prequential accuracy and runtime\footnote{https://orangedatamining.com}. 
\par
The statistical test results on average prequential accuracy shows that our model statistically significantly outperforms the baselines. In terms of runtime, KUE has the lowest rank, but our model is successful in improving its accuracy by 15.80\%.
 Furthermore, unlike the baselines, our model's runtime is not affected by the size of the feature set, and has a competitive runtime for stream mining. 
\begin{figure}[h] 
    \centering
    \fbox{

        \begin{minipage}{1\columnwidth}
            \centering\includegraphics[width=0.9\columnwidth]{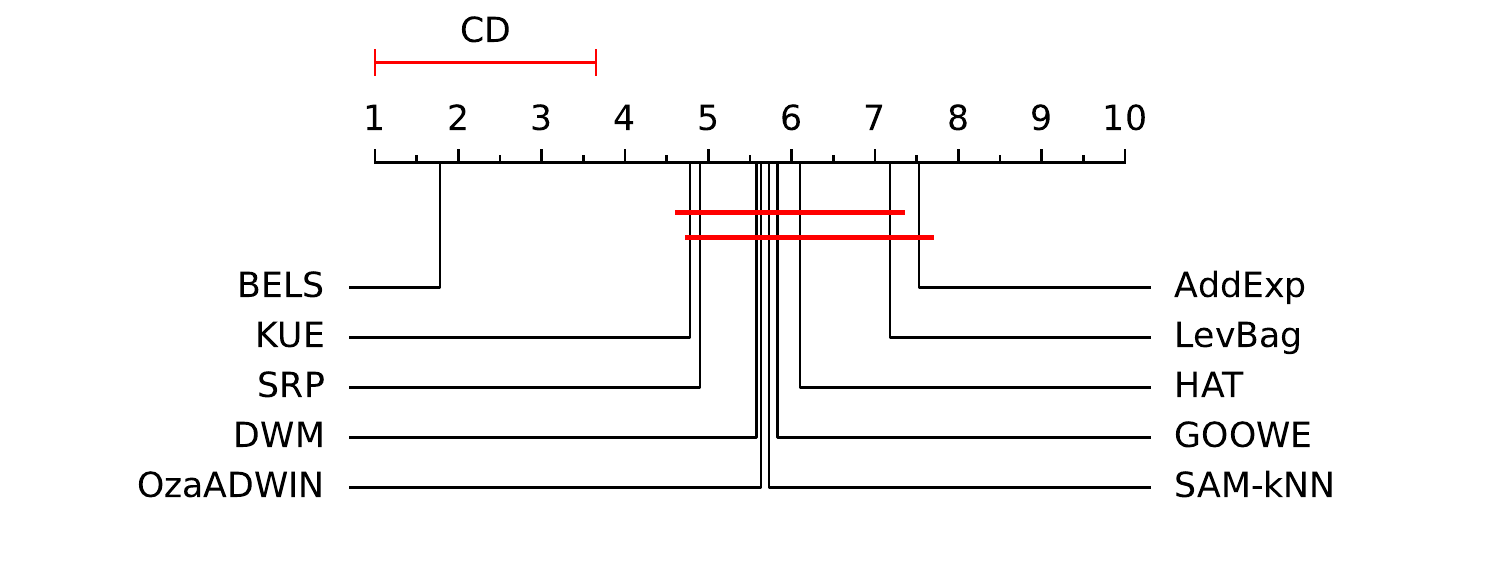}\\
            (a) Average prequential accuracy
        \end{minipage}
        }

         \fbox{
        \begin{minipage}{1\columnwidth}
            \centering\includegraphics[width=0.9\columnwidth]{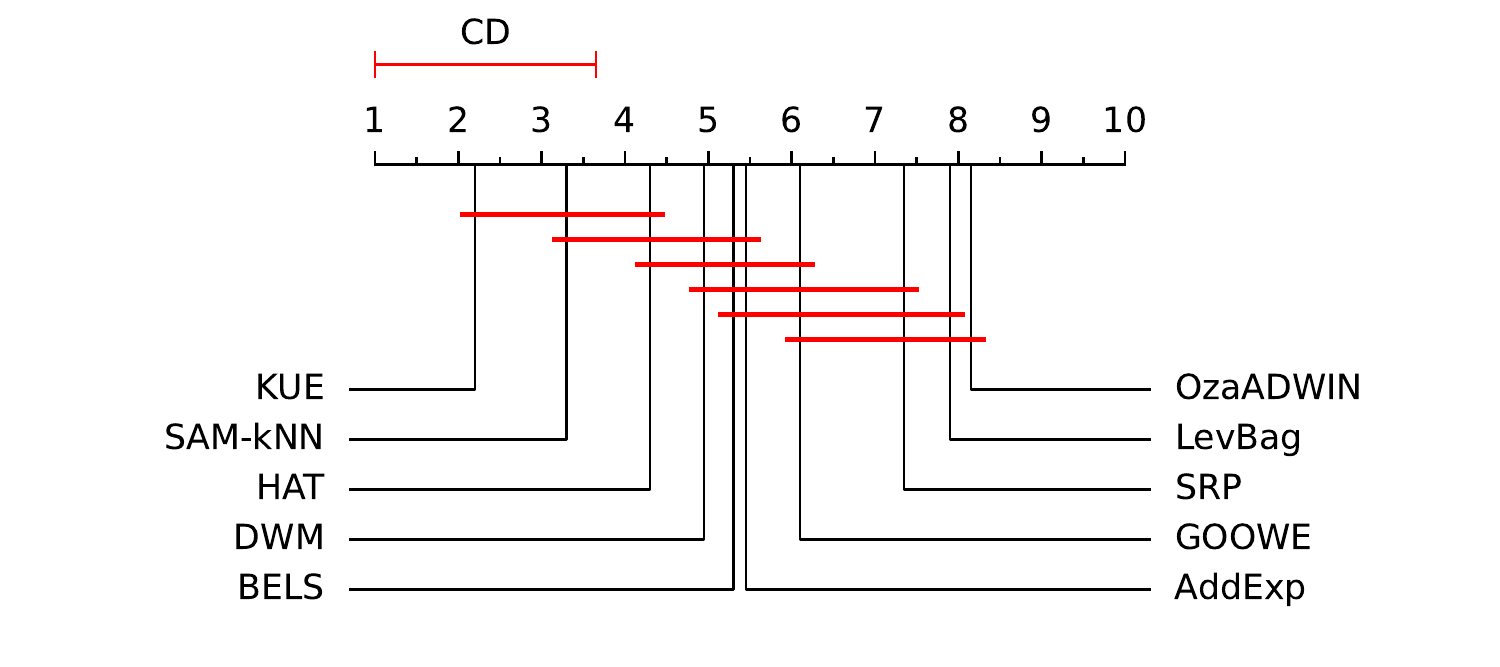}\\
            (b) Runtime
        \end{minipage}
        
    }
    \caption{Critical distance diagrams for the (a) Average prequential accuracy and (b) runtime using the data on Tables \ref{tab:acc_results} and \ref{tab:time_results}. (CD=2.65)}
    \label{fig:stat}
\end{figure}    

\newcolumntype{M}[1]{>{\centering\arraybackslash}m{#1}}

\newcolumntype{P}[1]{>{\centering\arraybackslash}p{#1}}

\renewcommand{\arraystretch}{1.3}
\begin{table*}[t]
    \centering
    \caption{{Hyperparameter sensitivity analysis. Results are reported in average prequential accuracy and runtime (in centiseconds for processing 1,000 data items). Best results for each hyperparameter for each dataset is in bold }}

    \resizebox{\textwidth}{!}{
    \begin{tabular}{|l l|r r r r r r r r r| r r r r r | r r r r r|}
         
    \hhline{~~-------------------}
    \multicolumn{2}{l|}{}&
 \multicolumn{9}{c|}{ $S_c$: chunk size}&
    
    \multicolumn{5}{c|}{ $M_o$: maximum number of output layer instance}&
    
    \multicolumn{5}{l|}{ $M_p$: maximim number of output layer instances in P}\\
        \hline
        Datasets (DT)& & 1&2 & 5 &10 & 20 & 50 & 100&250 &500& 5& 25 &50  &75 &150 &50 &100 &200 &300 &400 \\
        \hline

        Electricity (U)&Acc. & \textbf{88.18}&87.21&83.22&80.16&75.97&71.90&69.30&72.13&71.22&85.27&86.67&86.92&\textbf{87.21}&87.20&86.85&87.20&\textbf{87.21}&\textbf{87.21}&\textbf{87.21}\\
        &Runtime& 1,195.20 &597.06& 326.18&166.44&95.71&46.72&28.44&26.35&\textbf{11.78}&\textbf{199.95}&322.89&460.49&597.06&482.74 &\textbf{448.57}&584.87&585.32&597.06&522.78\\
        \hline
        Usenet (A \& R)&  Acc.& 68.79&\textbf{76.50}&72.07&71.14&70.50&70.29&63.536& -& -&65.35&\textbf{70.29}&\textbf{70.29}&\textbf{70.29}&\textbf{70.29}&\textbf{70.29}&\textbf{70.29}&\textbf{70.29}&\textbf{70.29}&\textbf{70.29}\\
        &Ruuntime& 1,304.66&536.66&293.33&112.00&48.66&21.33&\textbf{14.53}& -& - & \textbf{19.86}&21.52&20.36&\textbf{20.36}&\textbf{20.36}&\textbf{20.36}&\textbf{20.36}&\textbf{20.36}&\textbf{20.36}&\textbf{20.36}\\

        \hline
         MG2C2D (I \& G)& Acc. &81.19&91.42&92.45&92.71&92.84&\textbf{92.94}&92.67&92.50&92.76&92.35&92.75&92.82&\textbf{92.94}&\textbf{92.94}&92.87&\textbf{92.99}&92.94&92.94&92.94\\
          &Runtime&  9,470.50&5,967.25&2,850.06&1,348.71& 637.62&301.65&175.88&89.37&\textbf{55.86}&\textbf{50.24}&176.72&354.56&301.65&553.14&325.52&481.91&384.51&\textbf{301.65}&587.40\\
         \hline
        Rotating Hyperplane (I) &Acc.&  81.99&86.19&89.82&90.81&\textbf{91.00}&90.85&90.34&90.34&88.49&88.36&90.88&\textbf{90.94}&90.85&90.36&\textbf{91.05}&90.85&90.85&90.85&90.85\\
       & Runtime& 1,116.64&567.79&366.41&241.53&122.31&54.66&36.44&34.82&\textbf{24.91}&\textbf{12.06}&23.83&40.12&54.66&104.36&\textbf{39.73}&54.47&54.58&54.66&54.20\\
       
         \hline
    \end{tabular}
    }
    \label{tab:para_sens}
\end{table*}

\subsection{Hyperparameter Sensitivity Analysis}
In this part, we study the effects of the main hyperparameters on the overall performance of BELS. We use the same four datasets as in Section \uppercase\expandafter{\romannumeral5\relax}.A. With this intent, $S_c$ (chunk size), $M_p$ (maximum number of output layer instances in P), and $M_o$ (maximum number of output layer instances) are studied because of their importance in the learning process and handling the drifts. Table \ref{tab:para_sens} shows the results for hyperparameter sensitivity analysis. We first set a default setting for each dataset, based on the first 1,000 data instances (check Section \uppercase\expandafter{\romannumeral4\relax}.B for details). We only modify our target hyperparameter, and others stay the same in all of the experiments. Note that for the Usenet dataset, we do not report the results for chunk sizes 250 and 500, since the dataset has only 1,500 data items.
\par Based on the results for chunk size, we observe that the best accuracy results for each dataset is for the chunk sizes less than 100. BELS is designed for learning with small chunk sizes at every step, since it reacts to the changes in the accuracy, and frequently replaces the ensemble components with the new ones or the ones in the pool. Based on the observations in this section, we suggest using chunk sizes smaller than 100.  As mentioned earlier in Section \uppercase\expandafter{\romannumeral4\relax}.B, we use the first 1,000 instances to automatically determine the chunk size which is chosen from a set of $\{2,5,10,20,50\}$. The reason for removing chunk size equal to one from this set, is its inefficiency. As we can see in Table \ref{tab:para_sens}, runtime of the model with chunk size one, almost doubles the runtime of the model when chunk size two is used.
\par 
Based on the results in Section \uppercase\expandafter{\romannumeral5\relax}.A, we observe that having an ensemble of output layer instances improves the results, so the ensemble size should be greater than one. In our experiments, the default number of output layer instances ($M_o$) is set to 75. The reason is that adding more instances leads to longer runtime, and fewer output layer instances may result in poor performance in terms of accuracy.
\par
We also analyze the effect of pool size ($M_p$). After increasing  $M_p$ to 200 and more, the accuracy remains the same. The reason is that based on our ensemble method, after adding an output layer instance back to the learning process, we remove it from $P$ (see Algorithm 3:13). For this reason, usually $M_p$ does not exceed the defined limit for $P$; however, as we illustrated in Section \uppercase\expandafter{\romannumeral5\relax}.A, adding the pool has a positive impact on the performance of the model. We set the $M_p$ limit to 300 in all of the experiments. 
\par
Our suggestions for default hyperparameters is presented in Section \uppercase\expandafter{\romannumeral4\relax}.B, and they are repeated here for ease of reference: The chunk size is selected from a set of $\{2,5,10,20,50\}$, $M_o$$=$ $75$, and $M_P$$=$ $300$. 
\renewcommand\citepunct{, }

\section{Conclusion and Future Work}
In this work, we present BELS, a novel ensemble model for data stream classification in non-stationary environments. We describe real-world and unique challenges that data stream causes and focus on handling the concept drift using a novel approach. BELS tracks changes in the accuracy of the ensemble components and frequently reacts to these changes by exchanging classifier components between the pool of removed ensemble components and the active ensemble members. 
The results show that on average, BELS improves the accuracy, compared to other state-of-the-art models designed specifically for data streams, by 13.28\%, and in terms of processing time, it is a suitable choice for evolving environments.
\par
Another problem in data stream environments is lack of available class labels in real world scenarios.
As a future work,
we plan to improve BELS for semi-supervised data stream classification.

\ifCLASSOPTIONcaptionsoff
  \newpage
\fi

%
\bibliographystyle{plainurl}
\bibliography{sample-base}
%
\par
\vskip -2\baselineskip plus -1fil 
\begin{IEEEbiography}
    [{\includegraphics[width=1in,height=1.25in,clip,keepaspectratio]{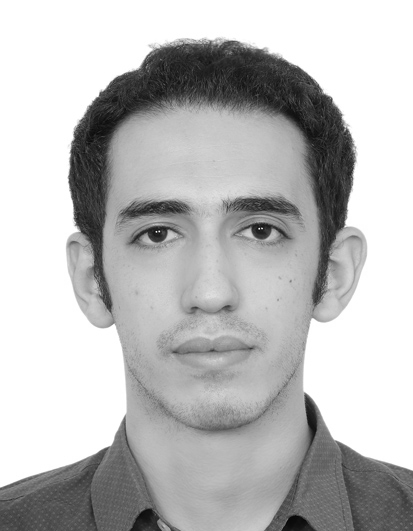}}]{Sepehr Bakhshi}
  received his B.S. degree in computer
engineering from University of Bonab, Bonab, Iran. In 2019, he joined the Bilkent Information Retrieval Group (BilIR). 
He conducted his MS research under the supervision of Prof. Fazlı Can at Bilkent University. His current research is focused on data stream mining, natural language processing and information retrieval.
\end{IEEEbiography}
\vskip -2\baselineskip plus -1fil 
\begin{IEEEbiography}
    [{\includegraphics[width=1in,height=1.25in,clip,keepaspectratio]{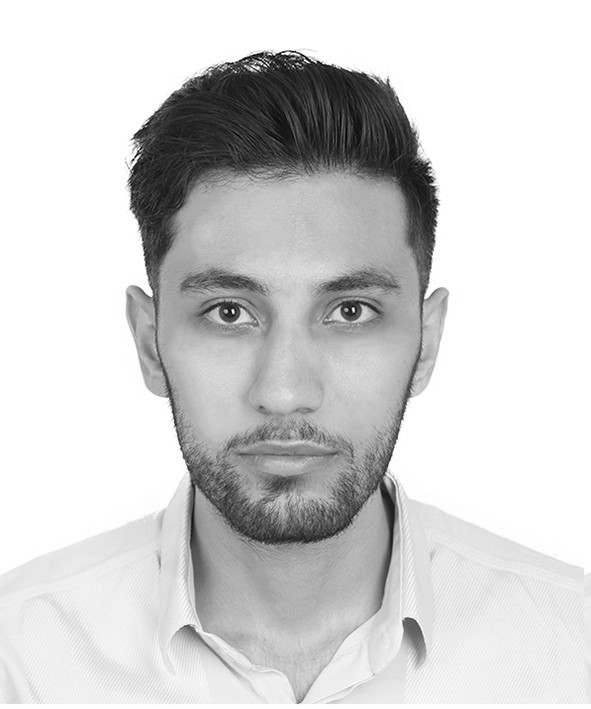}}]{Pouya Ghahramanian}
  Pouya Ghahramanian is a member of Bilkent Information Retrieval Group (BilIR). He conducted his MS research under the supervision of Prof. Fazlı Can at Bilkent University. He obtained his B.S. degree in computer software engineering from Iran University of Science and Technology (IUST) in 2018. His research interests include machine learning, natural language processing  and data stream mining.
\end{IEEEbiography}
\vskip -2\baselineskip plus -1fil 
\begin{IEEEbiography}
    [{\includegraphics[width=1in,height=1.25in,clip,keepaspectratio]{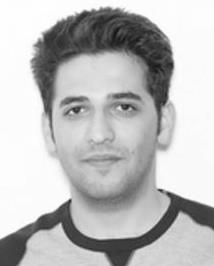}}]{Hamed Bonab} received the Ph.D. degree in computer science from the College of Information and
Computer Sciences, University of Massachusetts Amherst, Amherst, MA, USA. During his Ph.D. he worked with Prof. James Allan in the Center for Intelligent Information Retrieval (CIIR). He obtained the B.S. degree in computer engineering from the Iran University of Science and Technology, Tehran, Iran, and the M.S. degree in computer engineering from Bilkent University, Ankara, Turkey. His current research interests include information retrieval, natural language processing, machine learning, and stream processing. He is currently an Applied Scientist II at Amazon.com Inc. working on various product search problems.
\end{IEEEbiography}
\vskip -2\baselineskip plus -1fil 
\begin{IEEEbiography}
    [{\includegraphics[width=1in,height=1.25in,clip,keepaspectratio]{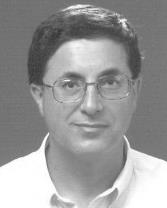}}]{Fazli Can}
 received the B.S. and M.S. degrees in
electrical-electronics and computer engineering
and the Ph.D. degree in computer engineering from
Middle East Technical University, Ankara, Turkey,
in 1976, 1979, and 1985, respectively. He conducted
his Ph.D. research under the supervision of Prof. E.
Ozkarahan; at Arizona State University, Tempe, AZ,
USA, and Intel, Chandler, AZ, USA; as a part of the
RAP Database Machine Project.
He is currently a Faculty Member at Bilkent University, Ankara. Before joining Bilkent, he was a
tenured Full Professor at Miami University, Oxford, OH, USA. He co-edited
ACM SIGIR Forum from 1995 to 2002 and is a Co-Founder of the Bilkent
Information Retrieval Group (BilIR), Bilkent University. His interest in dynamic
information processing dates back to his 1993 incremental clustering paper in
ACM Transactions on Information Systems and some other earlier work with
Prof. E. Ozkarahan on dynamic cluster maintenance. His current research
interests include information retrieval and data mining.
\EOD
\end{IEEEbiography}

\end{document}